\documentclass[preprint]{imsart}

\usepackage[nocompress]{cite}

\usepackage[pdftex]{graphicx}
\graphicspath{{figures/}{synthetic/}{dots/}{jain/}{stability/}{iris/}{asift/}{3d/}}
\DeclareGraphicsExtensions{.png}

\usepackage{algorithm, algorithmic}
\usepackage{url}
\usepackage{array}
\usepackage[caption=false,font=normalsize,labelfont=sf,textfont=sf]{subfig}
\usepackage{fixltx2e}
\usepackage{rotating}
\usepackage{tabularx}
\usepackage{multirow}

\usepackage{tepper}

\usepackage{mathrsfs}
\newcommand{\deterministic}[1]{#1}
\newcommand{\random}[1]{\mathscr{#1}}
\newcommand{\father}[1]{#1_F}
\newcommand{\randomEdge}{\gamma}

\begin{document}

\begin{frontmatter}
\title{Meaningful Clustered Forest: an Automatic and Robust Clustering Algorithm}
\runtitle{Meaningful Clustered Forest}

\begin{aug}
  \author{
    Mariano Tepper
    \ead[label=e1]{mtepper@dc.uba.ar}
  }

  \address{
    Departamento de Computaci\'on, Facultad de Ciencias Exactas y Naturales, Universidad de Buenos Aires\\
    \printead{e1}
  }

  \author{
    Pablo Mus\'e
    \ead[label=e3]{pmuse@fing.edu.uy}%
  }

  \address{
    Instituto de Ingenier\'ia El\'ectrica, Facultad de Ingenier\'ia, Universidad de la Rep\'ublica,\\
    \printead{e3}
  }

  \author{
    Andr\'es Almansa
    \ead[label=e4]{andres.almansa@telecom-paristech.fr}%
  }

  \address{
    CNRS - LTCI UMR5141, Telecom ParisTech\\
    \printead{e4}
  }

  \runauthor{M. Tepper et al.}

\end{aug}

\begin{abstract}
We propose a new clustering technique that can be regarded as a numerical method to compute the proximity gestalt. The method analyzes edge length statistics in the MST of the dataset and provides an a contrario cluster detection criterion. The approach is fully parametric on the chosen distance and can detect arbitrarily shaped clusters. The method is also automatic, in the sense that only a single parameter is left to the user. This parameter has an intuitive interpretation as it controls the expected number of false detections.
We show that the iterative application of our method can (1) provide robustness to noise and (2) solve a masking phenomenon in which a highly populated and salient cluster dominates the scene and inhibits the detection of less-populated, but still salient, clusters.
\end{abstract}


\end{frontmatter}

\maketitle

\section{Introduction}

Clustering is an unsupervised learning method that seeks to group observations into subsets (called clusters) so that, in some sense, intra-cluster observations are more similar than inter-cluster ones. Despite its intuitive simplicity, there is no general agreement on the definition of a cluster. In part this is due to the fact that the notion of cluster cannot be trivially separated from the context. Consequently, in practice different authors provide different definitions, usually derived from the algorithm being used, rather than the opposite.

Unfortunately, the lack of a unified definition makes it difficult to find a unifying clustering theory. A plethora of methods to assess or classify clustering algorithms have been developed, some of them with very interesting results. To cite a few~\cite{kleinberg02,kannan04,carlsson10}. For a broad perspective on clustering techniques, we refer the reader to the excellent overview of clustering methods recently reported by Jain~\cite{jain09}.



Human perception is extremely adapted to group similar visual objects. Based on psychophysical experiments using simple 2D figures, the Gestalt school studied the perceptual organization, and identified a set of rules that govern human perception~\cite{wertheimer38}. Each of these rules focuses on a single quality, or gestalt, many of which have been unveiled over the years.

One of the earlier and most powerful gestalts is proximity, which states that spatial or temporal proximity of elements may be perceived as a single group. Of course, the notion of distance is heavily embedded in the proximity gestalt. This is clearly illustrated in Figure~\ref{fig:proximity}. Two possible distances between the bars $B_1$ and $B_2$ that could be considered are
\begin{align*}
  d_M(B_1, B_2) &= \max_{\substack{p_1 \in B_1 \\ p_2 \in B_2}} ||p_1 - p_2|| \text{,}\\
  d_m(B_1, B_2) &= \min_{\substack{p_1 \in B_1 \\ p_2 \in B_2}} ||p_1 - p_2|| \text{.}
\end{align*}
In this particular example $|| \ldotp ||$ denotes the euclidean norm.
According to distance $d_M$, the bars are exactly at the same distance in both experiments, while according to distance $d_m$ the bars on the right are closer to each other. In this case, the distance $d_m$ seems to be more consistent with our perception.

\begin{figure}
  \centering
  \includegraphics[width=.5\textwidth]{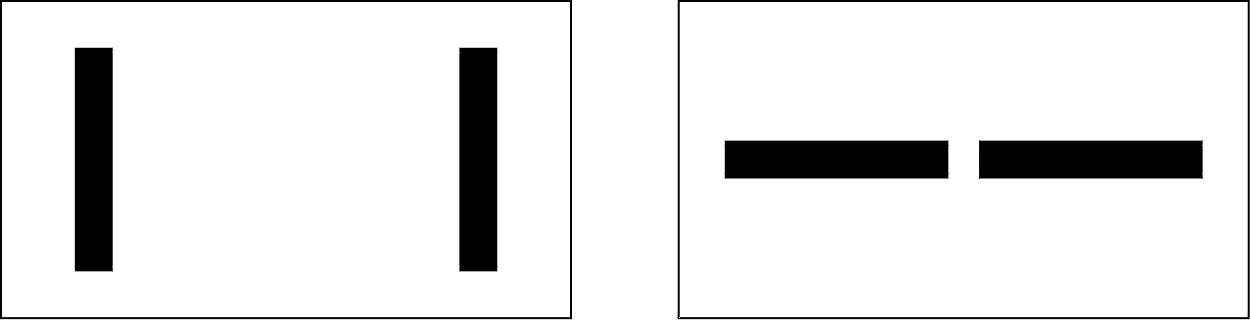}
  
  \caption{Two experiments with black bars. We perceive the bars on the left as more separated than the ones on the right. Nevertheless, there exists distances between sets that cannot capture the difference.}
  \label{fig:proximity}
\end{figure}

The conceptual grounds on which our work is based were laid by Zahn in a seminal paper from 1971~\cite{zahn71}. Zahn faced the problem of finding perceptual clusters according to the proximity gestalt and proposed three key arguments:
\begin{enumerate}
 \item \textbf{Only inter-point distances matter}. This imposes graphs as the only suitable underlying structure for clustering.
 \item \textbf{No random steps}. Results must remain stable for all runs of the detection process. In particular, random initializations are forbidden.
 \item \textbf{Independence from the exploration strategy}. The order in which points are analyzed must not affect the outcome of the algorithm.
\end{enumerate}
These conceptual statements, together with the preference for $d_m$ over $d_M$ or other distances between sets, led Zahn to use the Minimum Spanning Tree (MST) as a building block for clustering algorithms.
(The MST is the tree structure induced by the distance $d_m$~\cite{cormen}.)
Recently, psychophysical experiments performed by Dry~\etal~\cite{dry09} supported this choice. In these experiments individuals were asked to connect points of 30 major star constellations, to show the structure they perceive. Two examples of constellations are shown in Figure~\ref{fig:dryExperiment}. The outcome of these experiments was that, among five relational geometry structures, the MST and the Relative Neighborhood Graph (RNG) exhibit the highest degree of agreement with the empirical edges. In the RNG, two points $p$ and $q$ are connected by an edge whenever there does not exist a third point $r$ that is closer to both $p$ and $q$ than they are to each other. The MST is a subgraph of the RNG. Nonetheless the diagonal variance of both groups might suggest that sometimes other links not present nor in the MST nor in the RNG are used.

From a theoretical point of view, Carlsson and M\'emoli~\cite{carlsson10} proved very recently that the single-link hierarchy (i.e. a hierarchical structure built from the MST, as explained later) is the most stable choice and has good convergence properties among other hierarchical techniques.

\begin{figure}
  \centering{
  \begin{tabular}{cp{.2in}c}
    \begin{minipage}{.45\textwidth}
    \begin{tabular}{@{\hspace{0pt}}c@{\hspace{4pt}}c@{\hspace{0pt}}}
      Cetus & Draco \\
      \includegraphics[width=.49\textwidth]{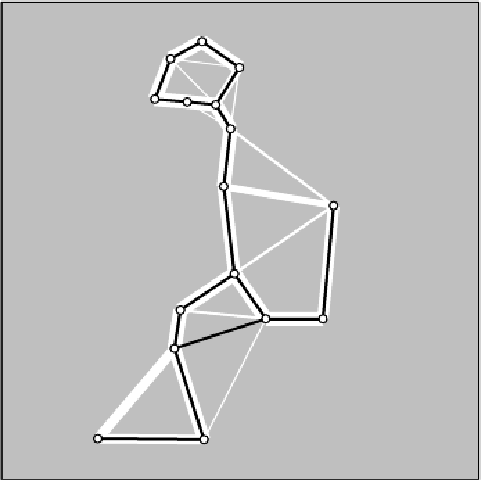} & \includegraphics[width=.49\textwidth]{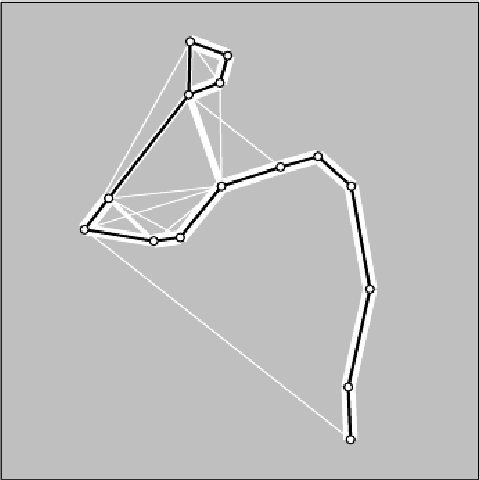}
    \end{tabular}
    \end{minipage} & &
    \begin{minipage}{.35\textwidth}
    \includegraphics[width=\textwidth]{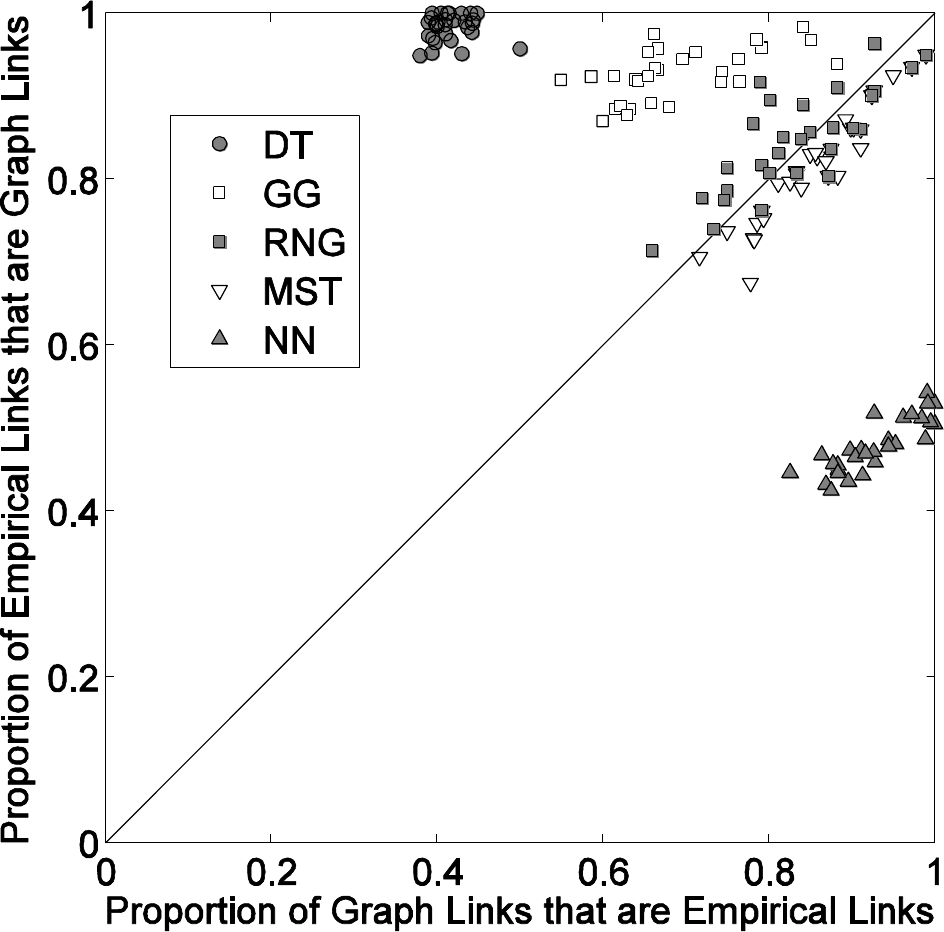}
    \end{minipage}
  \end{tabular}
  }
  \caption{Left and middle: example constellations shown in black and the aggregated empirical structure shown in white. The number of persons that chose an edge is represented by the edge's width. Right: proportional overlap between graph and empirical structure links for Delaunay Triangulation (DT), Gabriel Graph (GG), Relative Neighborhood graph (RNG), Minimum Spanning Tree (MST), and Nearest Neighbors (NN). Each data point represents one of the 30 stimuli. Reproduced from~\cite{dry09}.}
  \label{fig:dryExperiment}
\end{figure}

Zahn~\cite{zahn71} suggested to cluster a feature set by eliminating the inconsistent edges in the minimum spanning tree. That is, instead of constructing a MST and as a consequence of the eliminations, a minimum spanning forest is built.

Since then, variations of the limited neighborhood set approaches have been extensively explored. The criteria in most works are based on local properties of the graph. Since perceptual grouping implies an assessment of local properties versus global properties, exclusively local methods must be discarded or patched. For example, Felzenszwalb and Huttenlocher~\cite{felzenszwalb04} and Bandyopadhyay~\cite{bandyopadhyay04} make use of the MST and RNG respectively. However, in order to correct local observations and to produce a reasonable clustering, they are forced to consider additional ad hoc global criteria.

The computation of the MST requires previous computation of the complete graph. This is a major disadvantage of MST-based clustering methods, that impose severe restrictions both on time and memory. The obvious workaround is to prune a priori the complete graph (e.g. in image segmentation, the image connexity is exploited), but unfortunately it might produce artifacts in the final solution. In a recent work Tepper~\etal~\cite{tepper11mst} proposed an efficient method to compute the MST on metric datasets. The use of this method allows for a significant performance boost over previous MST-based methods (e.g.~\cite{cao06unified, felzenszwalb04}), thus permitting to treat large datasets.

From an algorithmic point of view, the main problem with the Gestalt rules is their qualitative nature. Desolneux~\etal developed a detection theory which seeks to provide a quantitative assessment of gestalts~\cite{desolneux08}. This theory is often referred as Computational Gestalt Theory and it has been successfully applied to numerous gestalts and detection problems~\cite{cao2005, grompone10, rabin09}. It is primarily based on the Helmholtz principle which states that no structure is perceived in white noise. In this approach, there is no need to characterize the elements one wishes to detect but contrarily, the elements one wishes to avoid detecting.

In the light of this framework, Desolneux~\etal analyzed the proximity gestalt, proposing a clustering algorithm~\cite{desolneux08}. It is founded on the idea that clusters are groups of points contained in a relatively small area. In other words, by counting points and computing the area that encloses them, one can assess the exceptionality of a given group of points.

The method proposed by Desolneux~\etal~\cite{desolneux08} suffers from some problems. First, it can only be applied to points in an Euclidean 2D space. Second, in order to compute the enclosing areas, the space has to be discretized a priori and such discretization is used to compute the enclosing areas; of course, different discretizations lead to different results. Last, two phenomena called collateral elimination and faulty union in~\cite{cao06unified} occur when an extremely exceptional cluster hides or masks other less but still exceptional ones. 

Cao~\etal~\cite{cao06unified} continued this line of research extending the clustering algorithm to higher dimensions and corrected the collateral elimination and faulty union issues, by introducing what they called indivisibility criterion. However, as their method is also based on counting points on a given region, it is still required that a set of candidate regions is given a priori. The set of test regions is defined to be a set $\Rect$ of hyper-rectangles parallel to the axes and of different edge lengths, centered at each data point. The choice of $\Rect$ is application specific since it is intrinsically related to cluster size/scale. For example, an exponential choice for the discretization of the rectangle space is made by Cao~\etal~\cite{cao06unified} introducing a bias for small rectangles (since they are more densely sampled). Then each cluster must be fitted by an axis-aligned hyper-rectangle $R \in \Rect$, meaning that clusters with arbitrary shapes are not detected. Even hyper-rectangular but diagonal clusters may be missed or oversegmented. A probability law modeling the number of points that fall in each hyper-rectangle $R \in \Rect$, assuming no specific structure in the data, must be known a priori or estimated. Obviously, this probability depends on the dimension of the space to be clustered.

Recently Tepper~\etal~\cite{tepper11graphClustering} introduced the concept of graph-based a contrario clustering. A key element in this method is that the area can be computed from a weighted graph, where the edge weight represents the distance between two points, using non-parametric density estimation. Since only distances are used, the dimensionality of the problem is reduced to one. However, since this method is conceived for complete graphs, it suffers from a high computational burden.

There is an additional concept behind clustering algorithms that was not stated before: a point, to belong to a cluster, must be similar to all points in the cluster or only to some of them? All the described region-based solutions imply choosing the first option since, in some sense, all distances within a group are inspected. Table~\ref{tab:clusteringCriteria} shows on which side some algorithms are. Since our goal is to detect arbitrarily shaped clusters, we must place ourselves in the second group. We can do this by using the MST.

\begin{table}
  \centering
  \begin{tabular}{p{2in}|p{2in}}
    \multicolumn{2}{c}{a point must be similar} \tabularnewline
    \centering{to all points in the cluster} &
    \centering{to at least one point in the cluster} \tabularnewline
    \hline
    $k$-means & single-link algorithm~\cite{fukunaga90} \tabularnewline
    Cao~\etal~\cite{cao06unified} & Mean Shift~\cite{comaniciu02} \tabularnewline
    Tepper~\etal~\cite{tepper11graphClustering} & Felzenszwalb and Huttenlocher~\cite{felzenszwalb04} \tabularnewline
  \end{tabular}
  \caption{Conceptually there are two different ways to form a cluster. To belong to a cluster a point must be similar to all points in the cluster or to at least one point in the cluster. All algorithms explicitly or implicitly chose one or the other.}
  \label{tab:clusteringCriteria}
\end{table}

Our goal is to design a clustering method that can be considered a quantitative assessment of the proximity gestalt. Hence in Section~\ref{sec:mstClustering} we propose a clustering method based on analyzing the distribution of distances of MST edges. The formulation naturally allows to detect clusters of arbitrary shapes. The use of trees, as minimally connected graphs, also leads to a fast algorithm. The approach is fully automatic in the sense that the user input only relates to the nature of the problem to be treated and not the clustering algorithm itself. Strictly speaking it involves one single parameter that controls the degree of reliability of the detected clusters. However, these methods can be considered parameter-free, as the result is not sensitive to the parameter value. Results on synthetic but challenging sets are presented in Section~\ref{sec:syntheticExamples}.

As the method relies on the sole characterization of non-clustered data, it is thus capable of detecting non-clustered data as such. In other words, in the absence of clustered data, the algorithm yields no detections. In Section~\ref{sec:stabilization} it is also shown that, by iteratively applying our method to datasets with clustered and unclustered data, it is possible to automatically separate both classes.

In Section~\ref{sec:masking}, we finally illustrate a masking phenomena where a highly populated cluster might occlude or mask less populated ones, showing that the iterative application of the MST-based clustering method is able to cope with this issue, thus solving very complicated clustering problems.

Results on three-dimensional examples of the complete process  are presented in Section~\ref{sec:3dPointClouds} and we expose some final remarks in Section~\ref{sec:mstClusteringConclusions}.

\section{A New Clustering Method: Proximal Meaningful Forest}
\label{sec:mstClustering}

We now propose a new method to find clusters in graphs that is independent from their shape and from their dimension. We first build a weighted undirected graph $G = (X, E)$ where $X$ is a set of features in a metric space $(M, d)$ and the weighting function $\omega$ is defined in terms of the corresponding distance function
\begin{equation}
  \omega((v_i, v_j)) = d(x_i, x_j) \textrm{.}
\end{equation}

\subsection{The Minimum Spanning Tree}

Informally, the Minimum Spanning Tree (MST) of an undirected weighted graph is the tree that covers all vertices with minimum total edge cost.

Given a metric space $(M, d)$ and feature set $X \subseteq M$, we denote by $G = (X, E)$ the undirected weighted graph where $E = X \times X$ and the graph's weighting function $\omega : E \rightarrow \R$ is defined as
\begin{equation}
  \omega((x_i, x_j)) = d(x_i, x_j) \quad \forall x_i, x_j \in X \textrm{.}
\end{equation}
The MST $T=(X, E_T)$ of the feature set $X$ is defined as the MST of $G$. A very important and classical property of the MST is that a hierarchy of point groups can be constructed from it.

\begin{notation}
  Let $T=(X, E_T)$ be the minimum spanning tree of $X$. For a group of points $\deterministic{C} \in X$, we denote
  \begin{equation}
    E({\deterministic{C}}) = \{ (v_i, v_j) \ |\ v_i, v_j \in C \land (v_i, v_j) \in E_T \}
  \end{equation}
  The edges in $E(\deterministic{C})$ are sorted in non-decreasing order, i.e.
  \begin{equation*}
    \forall\ e_i, e_j \in E({\deterministic{C}}) ,\ i < j \Rightarrow \omega(e_i) \leq \omega(e_j)
  \end{equation*}
\end{notation}

\begin{definition}
  Let $T=(X, E_T)$ be the minimum spanning tree of $X$. A component $C \subseteq X$ is a set such that the graph $G = (C, E(C))$ is connected and
  \begin{itemize}
    \item $\exists\ v \in V,\  C = \{v\}$ or
    \item $\nexists\ C' \in X,\  C \subset C' \quad \land \quad \omega_{\max}(C) > \omega_{\max}(C')$,
  \end{itemize}
  where $\displaystyle \omega_{\max}(C) = \max_{e \in E(C)} \omega(e)$.
  A single-link hierarchy $\Tree$ is the set of all possible components.
  \label{def:hierarchy}
\end{definition}

It is important to notice what the single-link hierarchy implies: given two components $C_1, C_2 \in \Tree$, it suffices that there exists a pair of vertices, one in $C_1$ and one in $C_2$ that are sufficiently near each other to generate a new component $C_F \in \Tree$, such that $C = C_1 \cup C_2$ and
\begin{equation}
  \omega_{\max} (C_F) = \min_{\substack{v_i \in C_1, v_j \in C_2 \\ (v_i, v_j) \in E_T}} \omega((v_i, v_j)) \text{.}
  \label{eq:maxFatherEqualsMinSiblings}
\end{equation}
An example is depicted in Figure~\ref{fig:mstAndLabels}. The direct consequence of this fact is that the use of the single-link hierarchy for clustering provides a natural way to deal with clusters of different shapes.

\begin{figure*}
  \centering
  \includegraphics[width=4in]{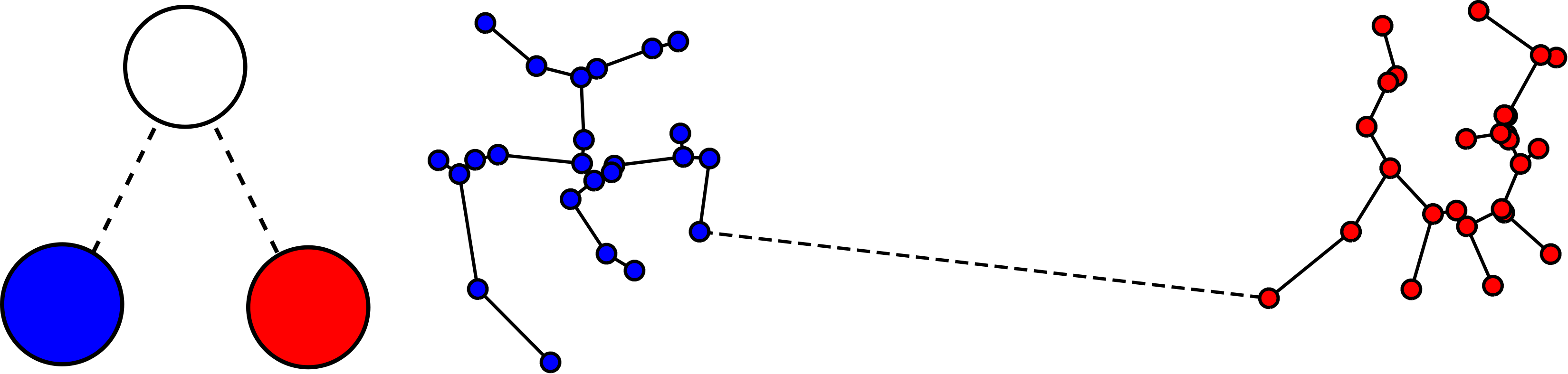}
  \put(-261, 54){$C_F$}
  \put(-283, 10){$C_1$}
  \put(-238, 10){$C_2$}
  \put(-130, 28){$\omega_{\max}(C_F)$}
  \put(-172, 0){$\omega_{\max}(C_1)$}
  \put(-90, 50){$\omega_{\max}(C_2)$}
  \put(-71, 45){\vector(1, -1){22}}
  \put(-174, 2){\vector(-1, 1){25}}
  \caption{Part of a minimal spanning tree. The blue node set and the red node set are linked by the dashed edge, creating a new node in the minimal spanning tree.}
  \label{fig:mstAndLabels}
\end{figure*}
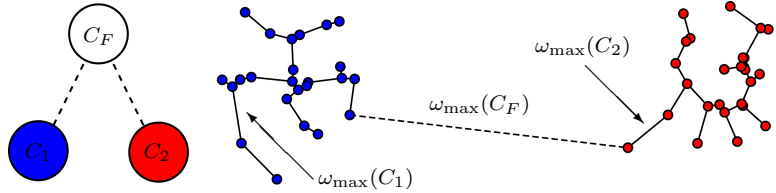

All minimum spanning tree algorithms are greedy. From Definition~\ref{def:hierarchy} and Equation~\ref{eq:maxFatherEqualsMinSiblings}, in the single-link hierarchy the component $C_F = C_1 \cup C_2$ is the father of $C_1$ and $C_2$ and
\begin{align}
  \omega_{\max}(C_F) &\geq \omega_{\max}(C_1) \label{eq:maxCF1} \\
  \omega_{\max}(C_F) &\geq \omega_{\max}(C_2) \label{eq:maxCF2} \textrm{.}
\end{align}
With the objective of finding a suitable partition and to the best of our knowledge, Felzenszwalb and Huttenlocher~\cite{felzenszwalb04} were the first to compare $\omega_{\max}(C_F)$ with $\omega_{\max}(C_1)$ and $\omega_{\max}(C_2)$, with an additional correction factor $\tau$. Components $C_1$ and $C_2$ are only merged if
\begin{equation}
  \min \Big[ \omega_{\max}(C_1) + \tau(C_1) ,\ \omega_{\max}(C_2) + \tau(C_2) \Big] \geq \omega_{\max}(C_F) \text{.}
  \label{eq:felzenszwalbCriteria}
\end{equation}
In practice $\tau$ is defined as $\tau(C) = s/|C|$ where $s$ plays the role of a scale parameter. The above definition presents a few problems. First, $\tau$ (i.e. $s$) is a global parameter and experiments show that clusters with different sizes and densities might not be recovered with this approach (Figure~\ref{fig:expDotsPedro}). Second, there is not an easy rule to fix $\tau$ or $s$ and, although it can be related with a scale parameter, there is no way to predict which specific value is best suited to a particular problem.

The exploration of similar ideas, while bearing in mind their shortcomings, leads us to a new clustering method.

\subsection{Proximal Meaningful Forest}

First, let us observe that the edge length distribution of an MST of a configuration of clustered points differs significantly from that of an unclustered point set (Figure~\ref{fig:expEdgesMST}). As a general idea, by knowing how to model unclustered data, one could detect clustered data by measuring some kind of dissimilarity between both.

In the a contrario framework, there is no need to characterize the elements one wishes to detect but contrarily the elements one wishes to avoid detecting (i.e. unclustered data).
To achieve such characterization we only need two elements: (1) a naive model and (2) a measurement made on structures to be potentially detected. The naive model is a probabilistic model that describes typical configurations where no structure is present and we will describe it in detail in Section~\ref{sec:backgroundModel}. We now focus on establishing a coherent measurement for MST-based clustering.

\begin{figure*}
  \centerline{
    \includegraphics[width=.24\textwidth]{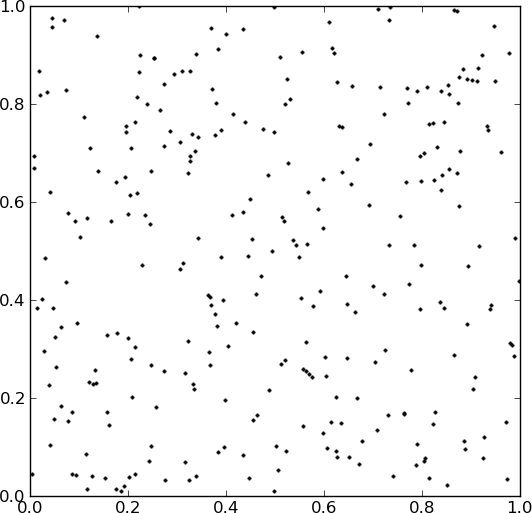}
    \includegraphics[width=.24\textwidth]{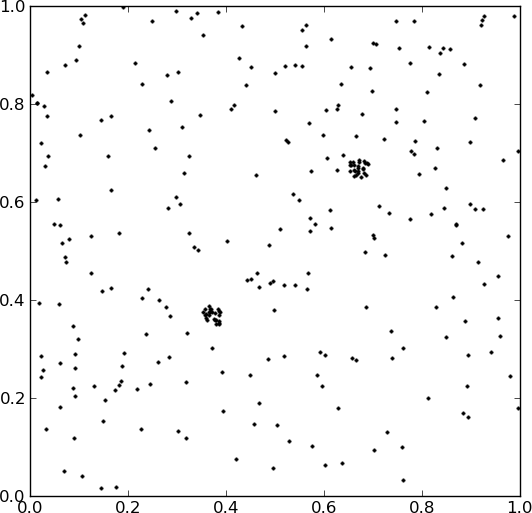}
    \includegraphics[width=.24\textwidth]{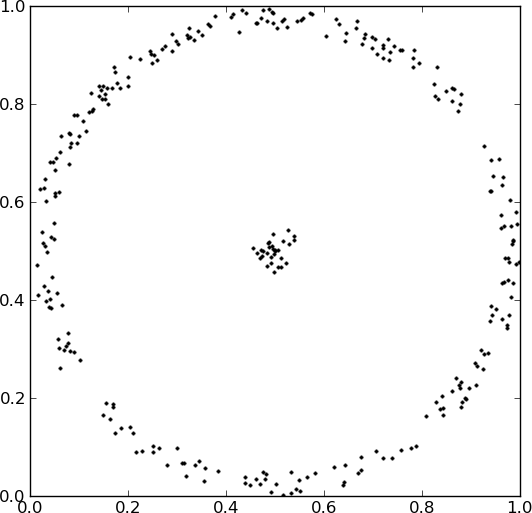}
    \includegraphics[width=.24\textwidth]{gaussians1.png}
  }
  \vspace{.1in}
  \centerline{
    \includegraphics[width=.24\textwidth]{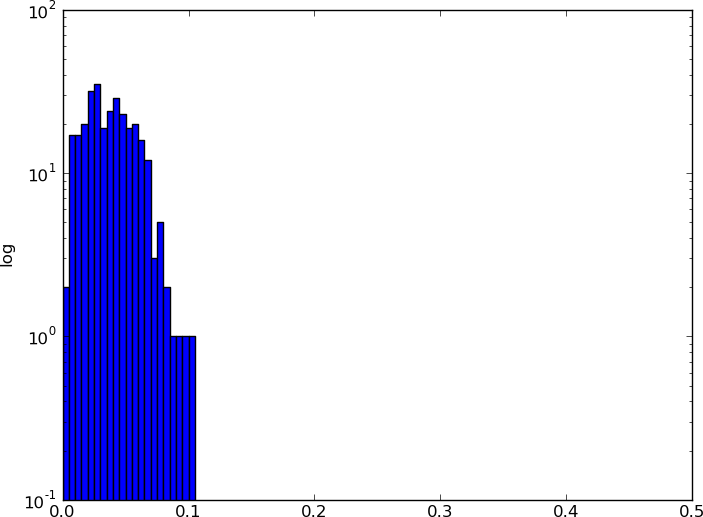}
    \includegraphics[width=.24\textwidth]{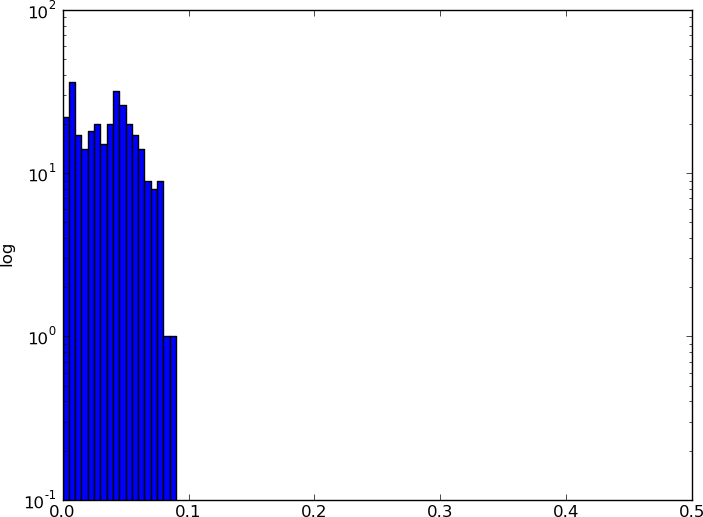}
    \includegraphics[width=.24\textwidth]{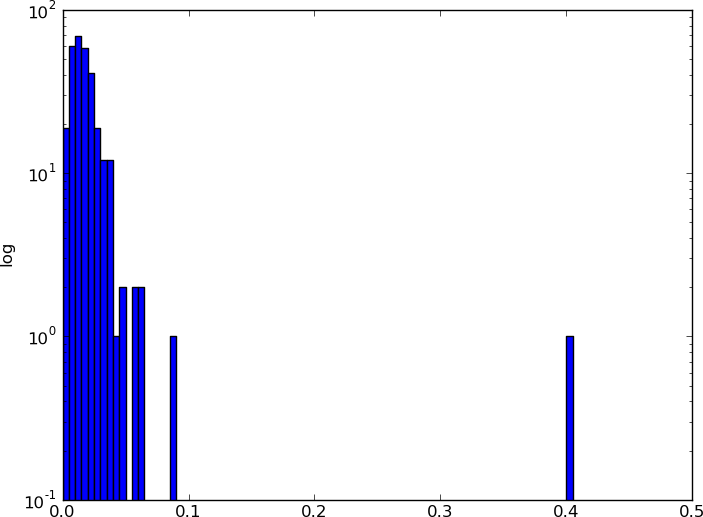}
    \includegraphics[width=.24\textwidth]{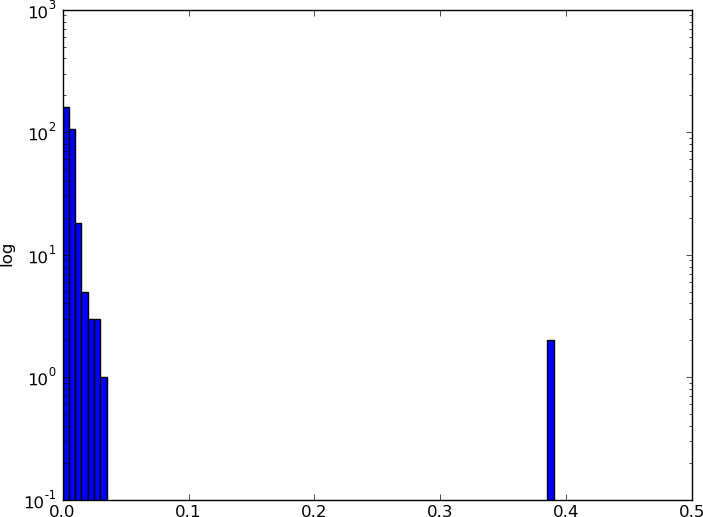}
  }
  
  \caption{Histograms (in logarithmic scale) of MST edge lengths from different points configurations. The non-clustered case (first column) differs from the other cases. Notice that clustered configurations also differ from each other.}
  \label{fig:expEdgesMST}
\end{figure*}

Concretely, we are looking to evaluate the probability of occurrence, under the background model (i.e. unclustered data), of a random set $\random{C}$ which exhibits the characteristics of a given observed set $\deterministic{C}$. Both sets have the same cardinality, i.e. $|E(\deterministic{C})| = |E(\random{C})| = K$.

The general principle has been previously explored. In 1983, following the same rationale Hoffman and Jain~\cite{hoffman83} proposed a similar idea: to perform a test of randomness. They built a null hypothesis using the edge length distribution of the MST and they performed a single test analyzing whether the whole dataset belongs to the random model or not by computing the difference between the theoretical and the empirical CDF. Jain~\etal~\cite{jain02} further refined this work, by using heuristic computations to separate the dataset into two or more subsets which were then tested using a two sample test statistic. Barzily~\etal recently continued this line of work~\cite{barzily09}. This approach introduces a bias towards the detection of compact (i.e. non-elongated) and equally sized clusters~\cite{barzily09}. From the perspective chosen in this work these characteristics can be seen as shortcomings, and thus we build a new method.

\begin{definition}
  Let $\mathcal{P}$ be a partition of $\R$.
  We define the equivalence relation $\sim$ where $\omega_1 \sim \omega_2$ if and only if $\exists P \in \mathcal{P}$ such that $\omega_1 \in P$ and $\omega_2 \in P$.
\end{definition}

We denote by $e_i$ the $i$-th edge of $E(\deterministic{C})$ and by $\randomEdge_i$ the $i$-th edge of $E(\random{C})$. Inspired by Equation~\ref{eq:felzenszwalbCriteria}, which proved successful as a decision rule to detect clusters, and associating it with Equations~\ref{eq:maxCF1} and~\ref{eq:maxCF2}, we compute
\begin{equation}
  \Pr \Big( \omega_{\max}(\random{C}) < \omega_{\max}(\deterministic{C}) \ |\ \omega_{\max}(\father{\random{C}}) \sim \omega_{\max}(\father{\deterministic{C}}) \Big).
\end{equation}

\begin{definition}
  Let $\deterministic{C} \in X$ be a component of the single-link hierarchy $\Tree$ induced by the minimum spanning tree $T = (X, E_T)$ such that $|\deterministic{C}| > 1$. We define the probability of false alarms (PFA) of $\deterministic{C}$ as
  \begin{equation}
    \PFA(\deterministic{C}) \stackrel{\mathrm{def}}{=} \Pr \Big( \omega_{\max}(\random{C}) < \omega_{\max}(\deterministic{C}) \ |\ \omega_{\max}(\father{\random{C}}) \sim \omega_{\max}(\father{\deterministic{C}}) \Big).
  \end{equation}
  \label{def:pfaMST}
\end{definition}
The constraint $|\deterministic{C}| > 1$ is needed since $E(\deterministic{C}) = \emptyset$ when $|\deterministic{C}| = 1$. Note that sets consisting of a single node must certainly not be detected. Conceptually, even when they are isolated, they constitute an outlier and not a cluster. We simply do not test such sets.

To detect unlikely dense subgraphs, a threshold is necessary on the PFA. In the classical \emph{a contrario} framework, a new quantity is introduced: the Number of False Alarms (NFA), i.e. the product of the PFA by the number of tested candidate clusters. The NFA has a more intuitive meaning than the PFA, since it is an upper bound on the expectation of the number of false detections~\cite{desolneux08}. The threshold is then easily set on the NFA.

\begin{definition}[Number of false alarms]
  We define the number of false alarms (NFA) of $C$ as
  \begin{equation}
    \NFA(C) \stackrel{\mathrm{def}}{=} (|X| - 1) \cdot \PFA(C).
  \end{equation}
  Notice that, by definition, $|X| - 1$ is the number of non-singleton sets in the single-link hierarchy.
  \label{def:nfaMST}
\end{definition}

\begin{definition}[Meaningful component]
  A component $C$ is \meps-meaningful if
  \begin{equation}
    \NFA(C) < \eps.
  \end{equation}

  \label{def:meaningfulComponent}
\end{definition}

In the following, it is important to notice a fact about the single-link hierarchy. The components are mainly determined by the sorted sequence of the edges from the original graph; this follows directly from Kruskal's algorithm~\cite{cormen}. However, the components are independent of the differences between the edges in that sorted sequence: only the order matters and not the actual weights of the edges.

We will now state a simple proposition that motivates Definition 4.
\begin{proposition}
  The expected number of \meps-meaningful clusters in a random single-link hierarchy (i.e. issued from the background model) is lower than \meps.
  \label{prop:expectation}
\end{proposition}

The proof is given in Appendix~\ref{sec:proofs} in the light of the discussion in the next section.
 
\subsection{The background model}
\label{sec:backgroundModel}

The distribution $\Pr \Big( \omega_{\max}(\random{C}) < \omega_{\max}(\deterministic{C}) \ |\ \omega_{\max}(\father{\random{C}}) \sim \omega_{\max}(\father{\deterministic{C}}) \Big)$ is not known a priori. Moreover, up to our knowledge there is no analytical expression for the cumulative edge distribution under $\Hy_0$ for the MST~\cite{hoffman83}.

We estimate this distribution by performing Monte Carlo simulations of the background process. However this estimation would involve extremely high computational costs.

We assume that the edge lengths in the cluster, given that $\omega_{\max}(\father{\random{C}}) \sim \omega_{\max}(\father{\deterministic{C}})$, are mutually conditionally independent and identically distributed. Let $\Omega$ be a random variable with this common distribution:
\begin{equation}
  \mathrm{F}_{\Omega} \Big(  \omega,\  \omega_{\max}(\father{\deterministic{C}}) \Big) = \Pr \Big( \Omega  \leq \omega \ |\ \omega_{\max}(\father{\random{C}})  \sim \omega_{\max}(\father{\deterministic{C}}) \Big) \text{.}
  \label{eq:distributionOmega}
\end{equation}
Then, using rank statistics:
\begin{multline}
  \Pr \Big( \omega_{\max}(\random{C})  \leq \omega_{\max}(\deterministic{C}) \ |\ \omega_{\max}(\father{\random{C}})  \sim \omega_{\max}(\father{\deterministic{C}}) \Big) \\
  =\mathrm{F}_{\Omega} \Big(  \omega_{\max}(\deterministic{C}),\  \omega_{\max}(\father{\deterministic{C}}) \Big)^K .
  \label{eq:rankStatistics}
\end{multline}
The distribution of $\Omega$ is much easier to compute than the one of $\omega_{\max}(\random{C})$ given that $\omega_{\max}(\father{\random{C}}) \sim \omega_{\max}(\father{\deterministic{C}}) \Big)$ as it requires fewer Monte Carlo simulations and thus we use it as in Equation~\ref{eq:rankStatistics} to compute the PFA.

Now, the independence hypothesis is clearly false even for i.i.d. graph vertices, because of the ordering structure and edge selection induced by the MST. However, the conditional dependence may be weak enough to make the model suitable in practice. This explains why naive classifiers, such as the Naive Bayes classifier, despite their simplicity, can often outperform more sophisticated classification methods~\cite{hastie01}. Therefore we only assume conditional independence in our model. Still, the suitability of this model has to be checked; a series of experiments shows that no clusters are detected in non-clustered, unstructured data (for instance, sets of independent, uniformly distributed vertices). On the other side of the problem, meaningful clusters always constitute a clear violation of this naive independence assumption.

The estimation process is depicted in Algorithm~\ref{algo:backgroundModel}.
Classically, one defines a point process and a sampling window. Hoffman and Jain~\cite{hoffman83} point out that the sampling window for the background point process is usually unknown for a given dataset. They use the convex hull arguing that it is the maximum likelihood estimator of the true sampling window for uniformly distributed two-dimensional data. In the experiments from Section~\ref{sec:syntheticExamples}, we simply use the minimum hiper-rectangle that contains the whole dataset as the sampling window.
However, there are problems whose intrinsic characteristics allow to define other background processes that do not involve a sampling window.

\begin{algorithm}
\caption{Compute $\Pr \Big( \omega_{\max}(\random{C}) < \omega_{\max}(\deterministic{C}) \ |\ \omega_{\max}(\father{\random{C}}) \sim \omega_{\max}(\father{\deterministic{C}}) \Big)$ for a set of $N$ points by $Q$ Monte Carlo simulations.}
\label{algo:backgroundModel}
\begin{algorithmic}

  \FORALL{$q$ such that $1 \leq q \leq Q$}
    \STATE $X \leftarrow$ draw $N$ points from the background point process.
    \STATE build the single-link hierarchy $\Tree_q$ from the MST of $X$.
  \ENDFOR
  
  \STATE compute a conditional histogram from the set $\lbrace \Tree_q \rbrace_{q = 1 \dots Q}$
  
\end{algorithmic}
\end{algorithm}

\subsection{Eliminating redundancy}
\label{sec:maximality}

While each meaningful cluster is relevant by itself, the whole set of meaningful components exhibits, in general, high redundancy: a meaningful component $C_1$ can contain another meaningful component $C_2$~\cite{cao06unified}. This question can be answered by comparing $\NFA(C_1)$ and $\NFA(C_2$) using Definition~\ref{def:meaningfulComponent}. The group with the smallest NFA must of course be preferred. Classically, the following rule
\begin{algorithmic}
  \FORALL{\meps-meaningful clusters $C_1$, $C_2$}
    \IF{$C_2 \subset C_1 \ \lor\  C_1 \subset C_2$}
      \STATE eliminate $\argmax \left( \NFA(C_1), \NFA(C_2) \right)$
    \ENDIF
  \ENDFOR
\end{algorithmic}
would have been used to perform the pruning of the set of meaningful components. Unfortunately, it leads to a phenomenon described in~\cite{cao06unified}, where it was called collateral elimination. A very meaningful component can hide another meaningful sibling, as illustrated in Figure~\ref{fig:collateralElimination}.

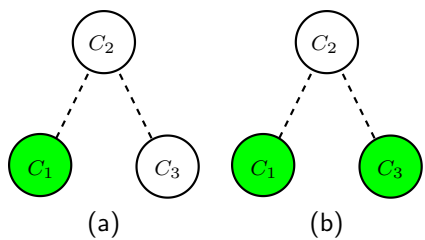
\begin{figure}
  \centering{
  \begin{tabular}{cc}
    \subfloat[]{\includegraphics[width=1in]{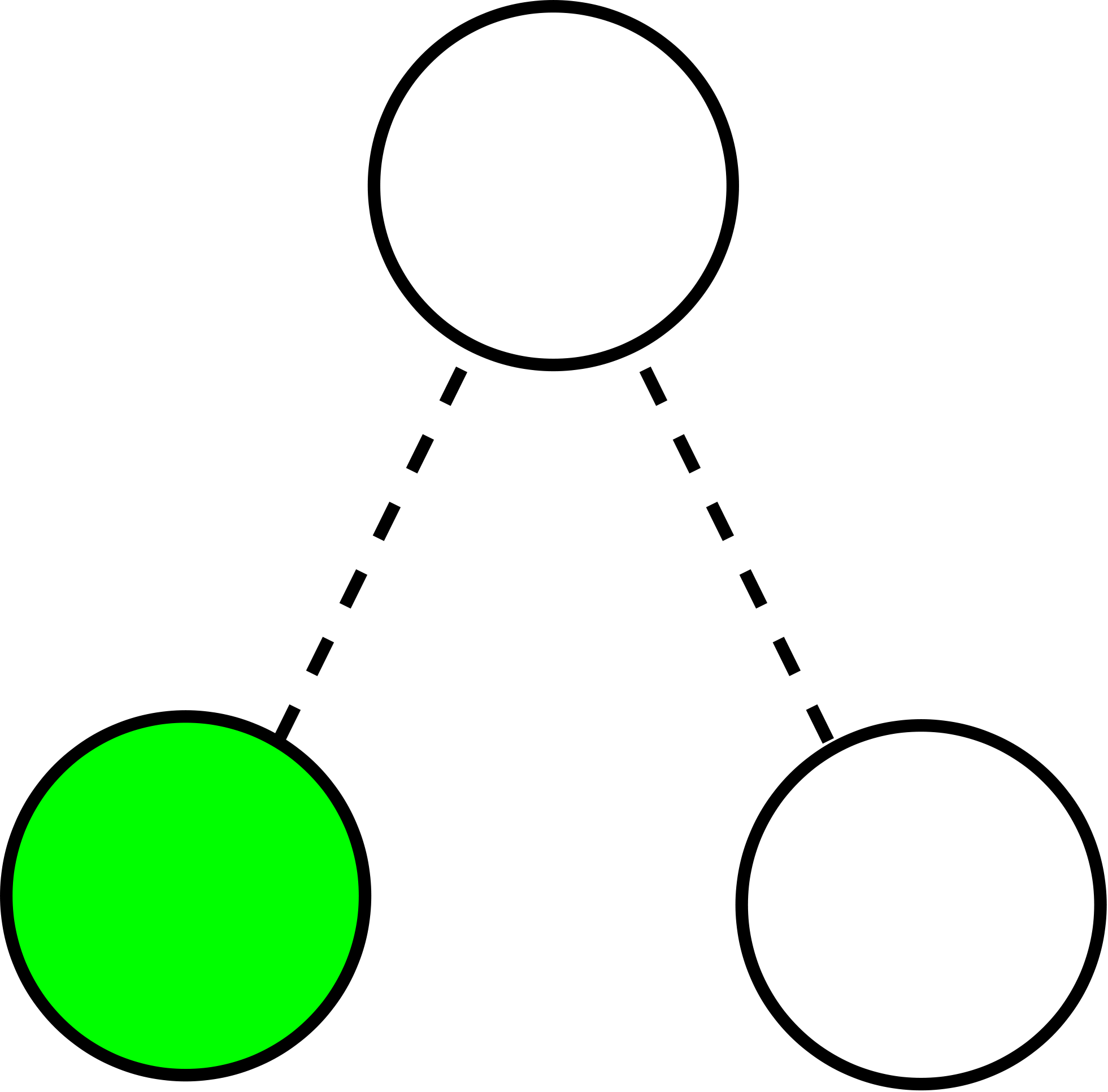}\label{fig:classicalMaximality}}
    \put(-42, 56){$C_2$}
    \put(-65, 9){$C_1$}
    \put(-17, 9){$C_3$}
    &
    \subfloat[]{\includegraphics[width=1in]{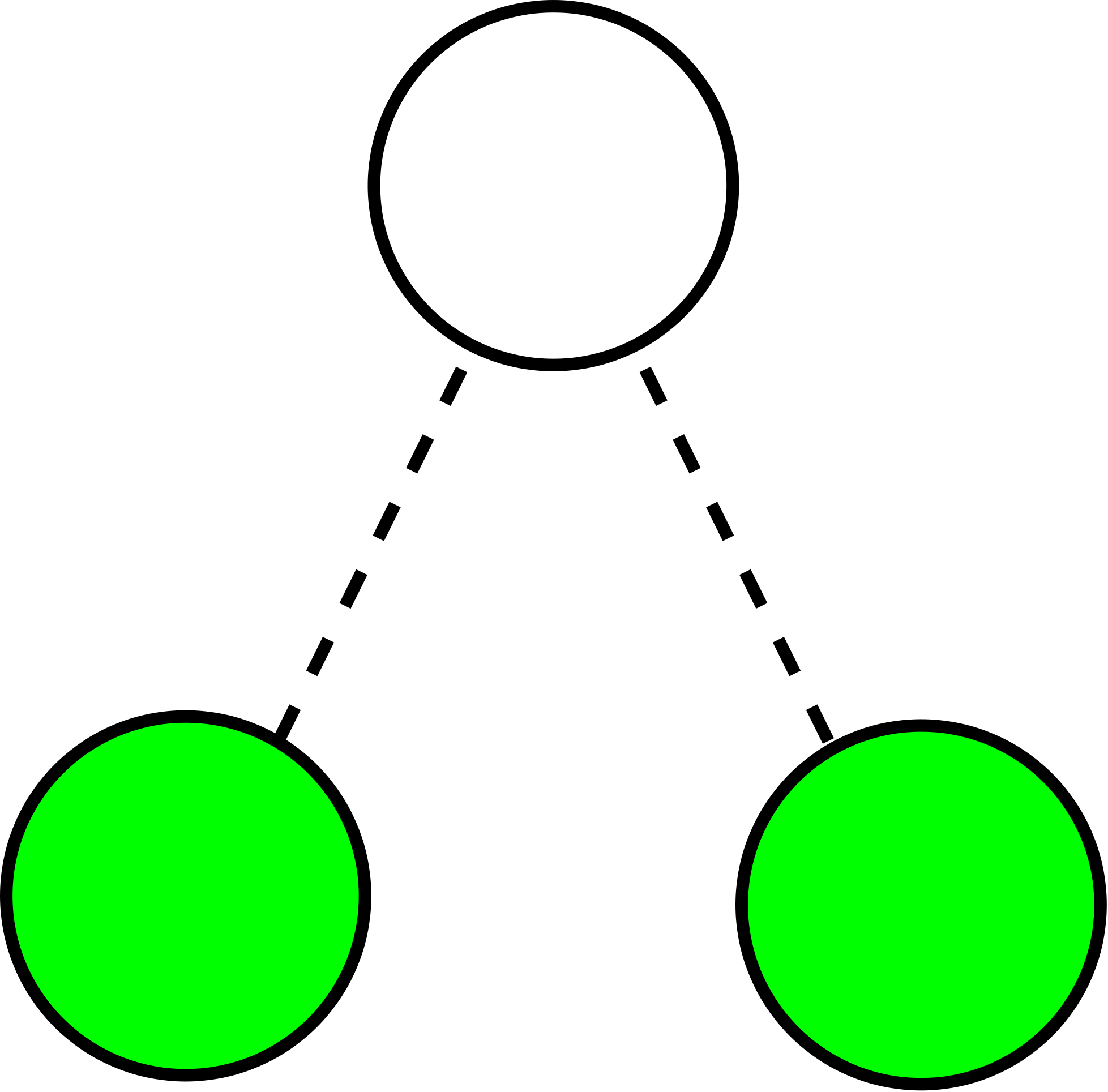}\label{fig:maximality}}
    \put(-42, 56){$C_2$}
    \put(-65, 9){$C_1$}
    \put(-17, 9){$C_3$}
  \end{tabular}
  }
  \caption{Example of collateral elimination. Three components $C_1, C_2, C_3$ such that $C_1 \subset C_2$, $C_3 \subset C_2$ and $\NFA(C_1) < \NFA(C_2) < \NFA(C_3) < \eps$. \subref{fig:classicalMaximality} The classical maximality rule only selects $C_1$ as a maximal component. \subref{fig:maximality} The scheme in Algorithm~\ref{algo:maximality} selects $C_1$ and $C_3$.}
  \label{fig:collateralElimination}
\end{figure}

The single-link hierarchy offers an alternative scheme to prune the redundant set of meaningful components, profiting from the inclusion properties of the dendrogram structure. It is non-other than the exclusion principle, defined first by Desolneux~\etal~\cite{desolneux08}, which states that

\begin{quote}
Let $A$ and $B$ be groups obtained by the same gestalt law. Then no point $x$ is allowed to belong to both $A$ and $B$. In other words each point must either belong to $A$ or to $B$.
\end{quote}
A simple scheme for applying the exclusion principle is shown in Algorithm~\ref{algo:maximality}.

Since we are choosing the components that are more in accordance with the proximity gestalt, we call the resulting components Proximal Meaningful Components (PMC). Then, we say that the set of all proximal meaningful components is a Meaningful Clustered Forest (MCF).

\begin{algorithm}
\caption{Eliminate redundant components from the set $\mathcal{M}$ of meaningful components.}
\label{algo:maximality}
\begin{algorithmic}[1]
  
  \STATE $\mathcal{F} \leftarrow \emptyset$
  
  \WHILE{$\mathcal{M} \neq \emptyset$}
    \STATE $\displaystyle C_{\min} \leftarrow \argmin_{C \in \mathcal{M}} \NFA(C)$
    \STATE eliminate $C_{\min}$ from $\mathcal{M}$
    \STATE eliminate all components $C$ from $\mathcal{M}$ such that $C \subset C_{\min}$
    \STATE eliminate all components $C$ from $\mathcal{M}$ such that $C_{\min} \subset C$
    \STATE add $C_{\min}$ to $\mathcal{F}$
  \ENDWHILE
  
  \STATE $\mathcal{M} \leftarrow \mathcal{F}$
  
\end{algorithmic}
\end{algorithm}


\section{Experiments on Synthetic examples}
\label{sec:syntheticExamples}

As a sanity check, we start by testing our method with simple examples. Figure~\ref{fig:expRings} present clusters which are well but not linearly separated. The meaningful clustered forest describes correctly the structure of the data.

\begin{figure*}
  \centering
  \begin{tabular}{cc}
    Input data & MCF \\
    \includegraphics[width=.3\textwidth]{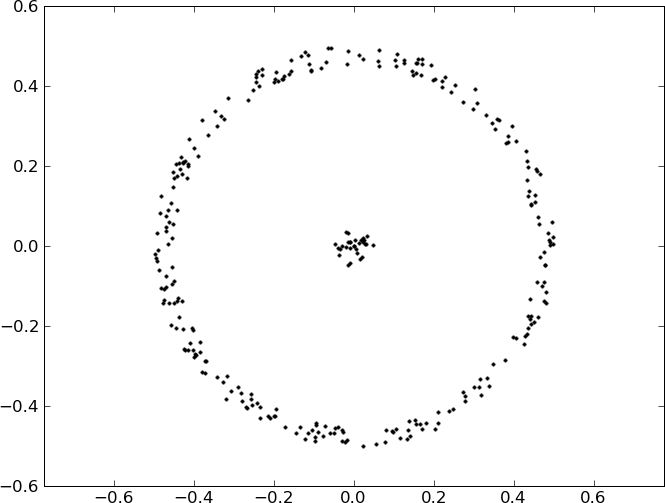} &
    \includegraphics[width=.3\textwidth]{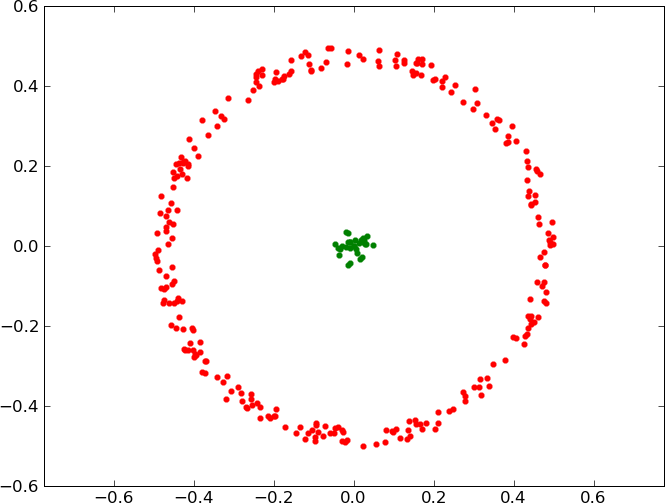} \\
    \includegraphics[width=.3\textwidth]{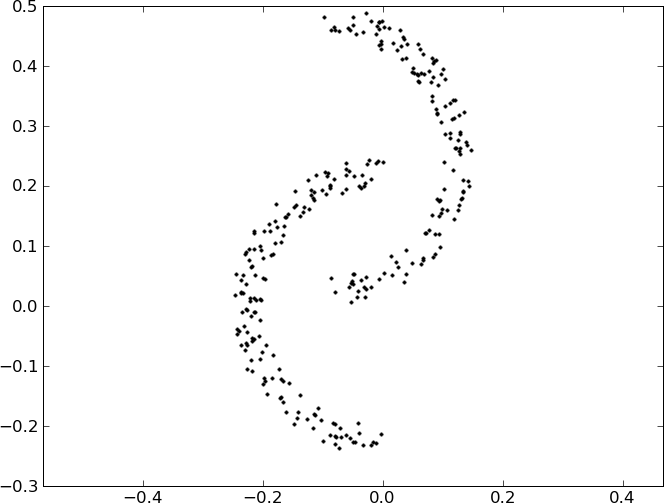} &
    \includegraphics[width=.3\textwidth]{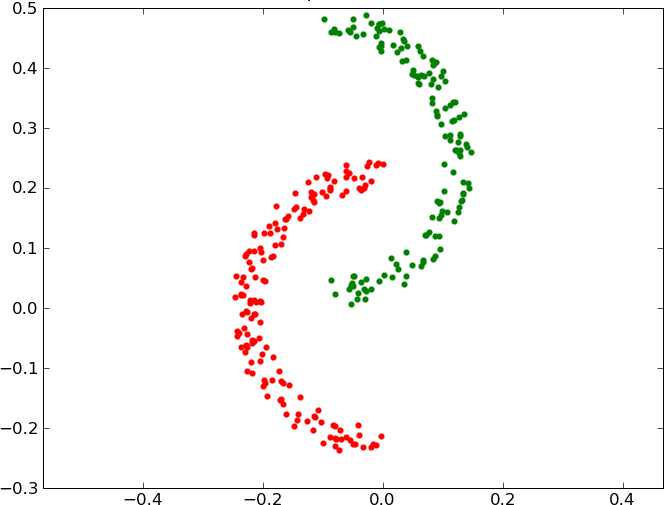}
  \end{tabular}
 
  \caption{The meaningful clustered forest correctly describes the points organization, even when clusters have arbitrary shapes.}
  \label{fig:expRings}
\end{figure*}

Figure~\ref{fig:expCao} shows an example of cluster detection in a dataset overwhelmed by outliers. The data consists of 950 points uniformly distributed in the unit square, and 50 points manually added around the positions $(0.4, 0.4)$ and $(0.7, 0.7)$. The figure shows the result of a numerical method involving the above NFA. The background distribution is chosen to be uniform in $[0, 1]^2$. Both visible clusters are found and their NFAs are respectively $10^{-15}$ and $10^{-8}$. Such low numbers can barely be the result of chance.

\begin{figure*}
  \centering{
    \begin{tabular}{cc}
      Input data & MCF \\
      \includegraphics[width=.3\textwidth]{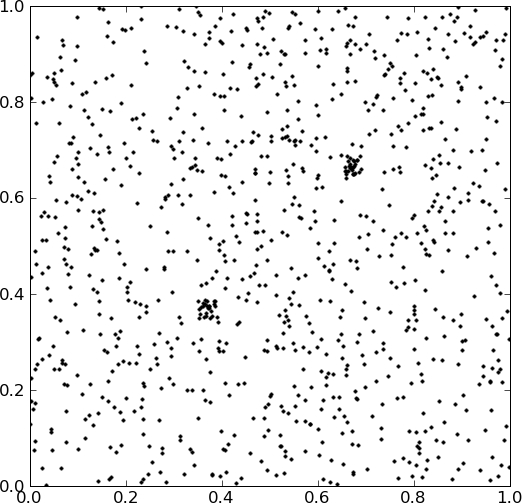} &
      \includegraphics[width=.3\textwidth]{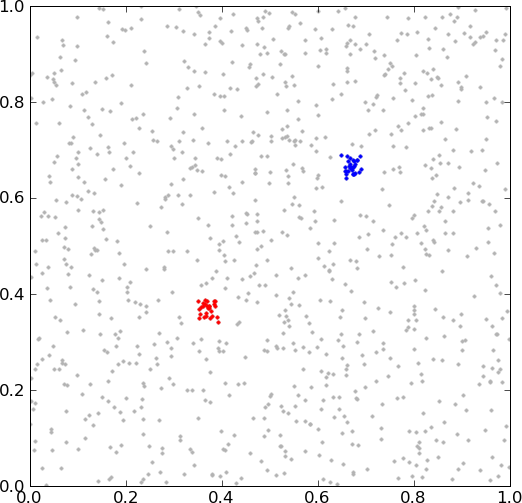}
    \end{tabular}
  }
 
  \caption{Similar experiment as performed by Cao~\etal in Figure 2~\cite{cao2005}. Clustering of twice 25 points around $(0.4, 0.4)$ and $(0.7, 0.7)$ surrounded by 950 i.i.d. points, uniformly distributed in the unit square. Exactly two proximal meaningful components are detected.}
  \label{fig:expCao}
\end{figure*}

The case of mixture of Gaussians, shown in Figure~\ref{fig:expGaussians}, provides an interesting example. On the tails, points are obviously sparser and the distance to neighboring points grows. Since we are looking for tight components, the tail might be discarded, depending on the Gaussian variance.

\begin{figure*}
  \centering{
    \begin{tabular}{cc}
      Input data & MCF \\
      \includegraphics[width=.3\textwidth]{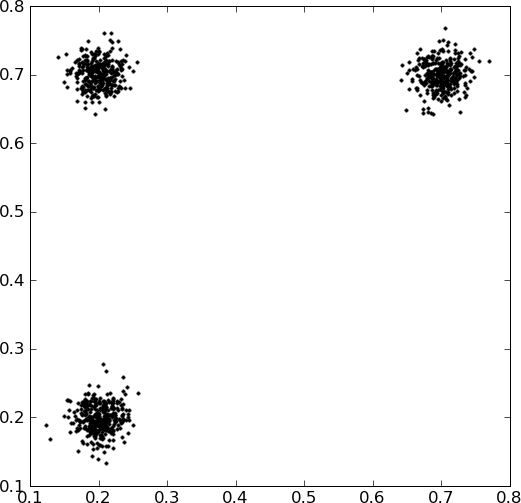} &
      \includegraphics[width=.3\textwidth]{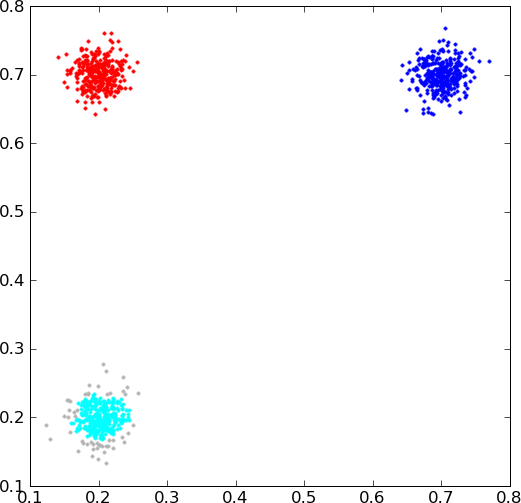} \\
    \end{tabular}
  }
  \caption{Clusters are correctly recovered in the mixture of three Gaussians. However some points are detected as noise (depicted in gray) in the tails.}
  \label{fig:expGaussians}
\end{figure*}

The example in Figure~\ref{fig:expDots} consists of a very complex scene, composed of clusters with different densities, shapes and sizes. Proximal components (i.e. we avoid testing $\NFA < \eps$) are displayed. Even when no decision about the statistical significance is made, the recovered clusters describe, in general, the scene accurately. Some oversplitting can be detected in proximal components. When a decision is made and only meaningful components are kept, we realize that the oversplit figures are not meaningful. As a sanity check, in Figure~\ref{fig:expDots-noise} we plot some of the detected structures superimposed to a realization of the background noise model. The input data in Figure~\ref{fig:dots} contains 764 points and for a given shape in it, with $W$ points, we plot the shape and $764-W$ points drawn from the background model. Among proximal components, the meaningful ones can be clearly perceived while non-meaningful ones are unnoticed.

\begin{figure*}

  \centerline{
    \subfloat[Input data]{
      \fbox{\includegraphics[width=.22\textwidth]{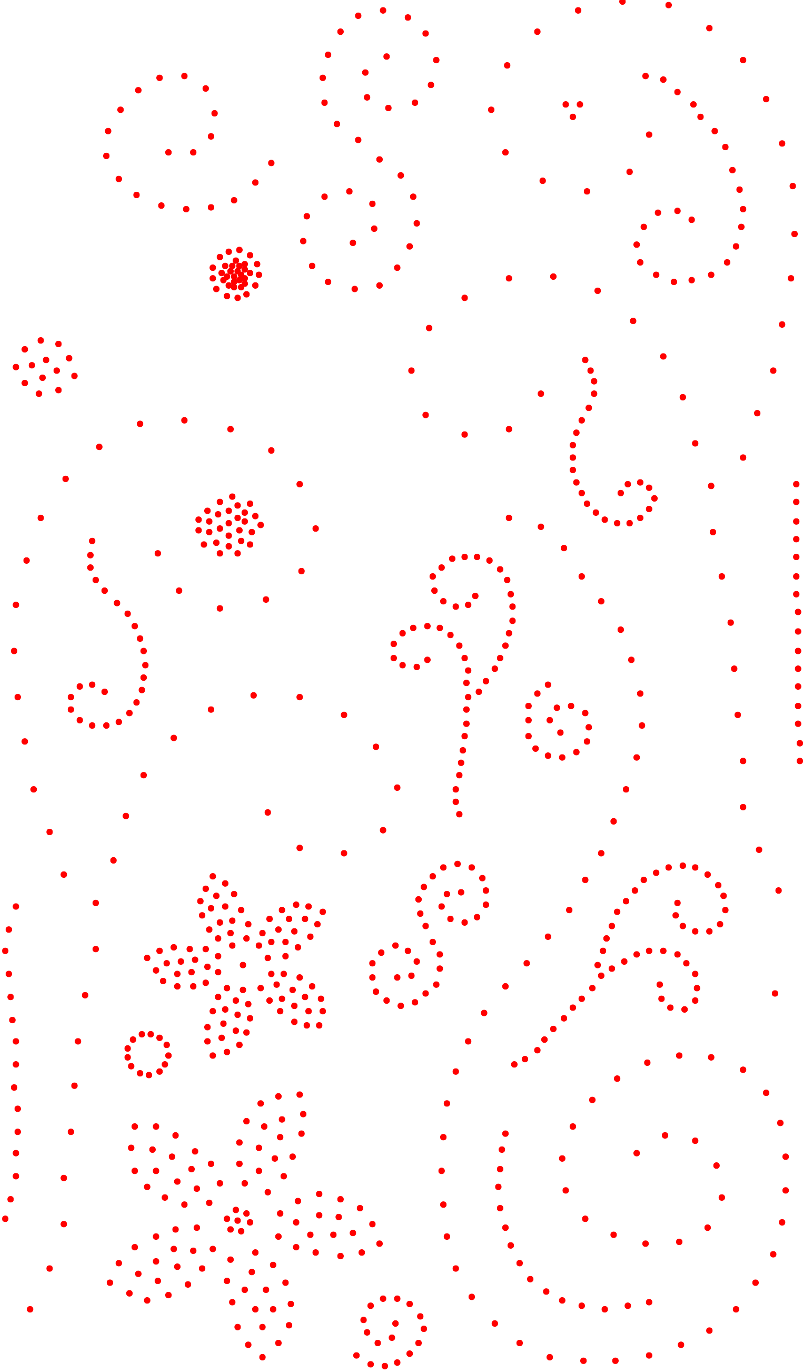}}
      \label{fig:dots}
    }
    \subfloat[MST]{
      \fbox{\includegraphics[width=.22\textwidth]{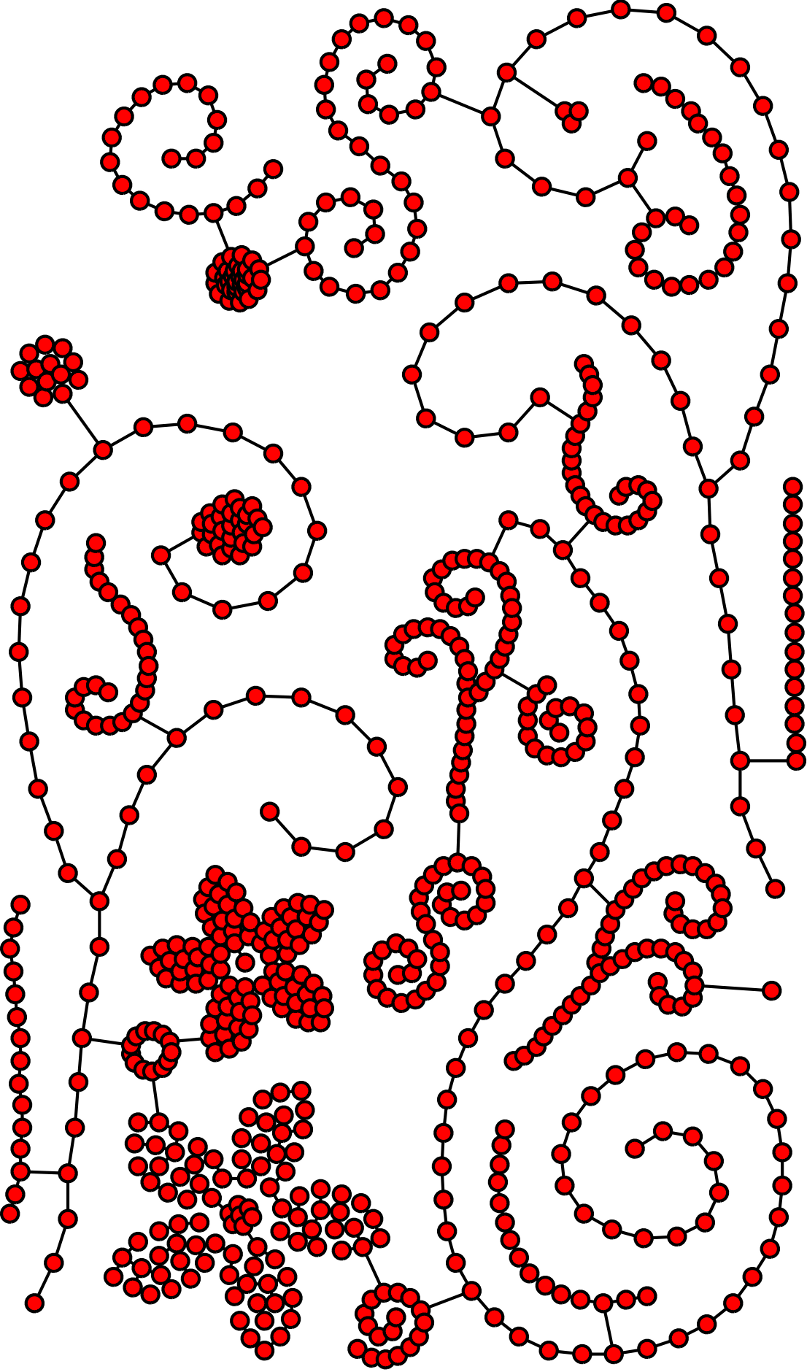}}
    }
    \subfloat[MC]{
      \fbox{\includegraphics[width=.22\textwidth]{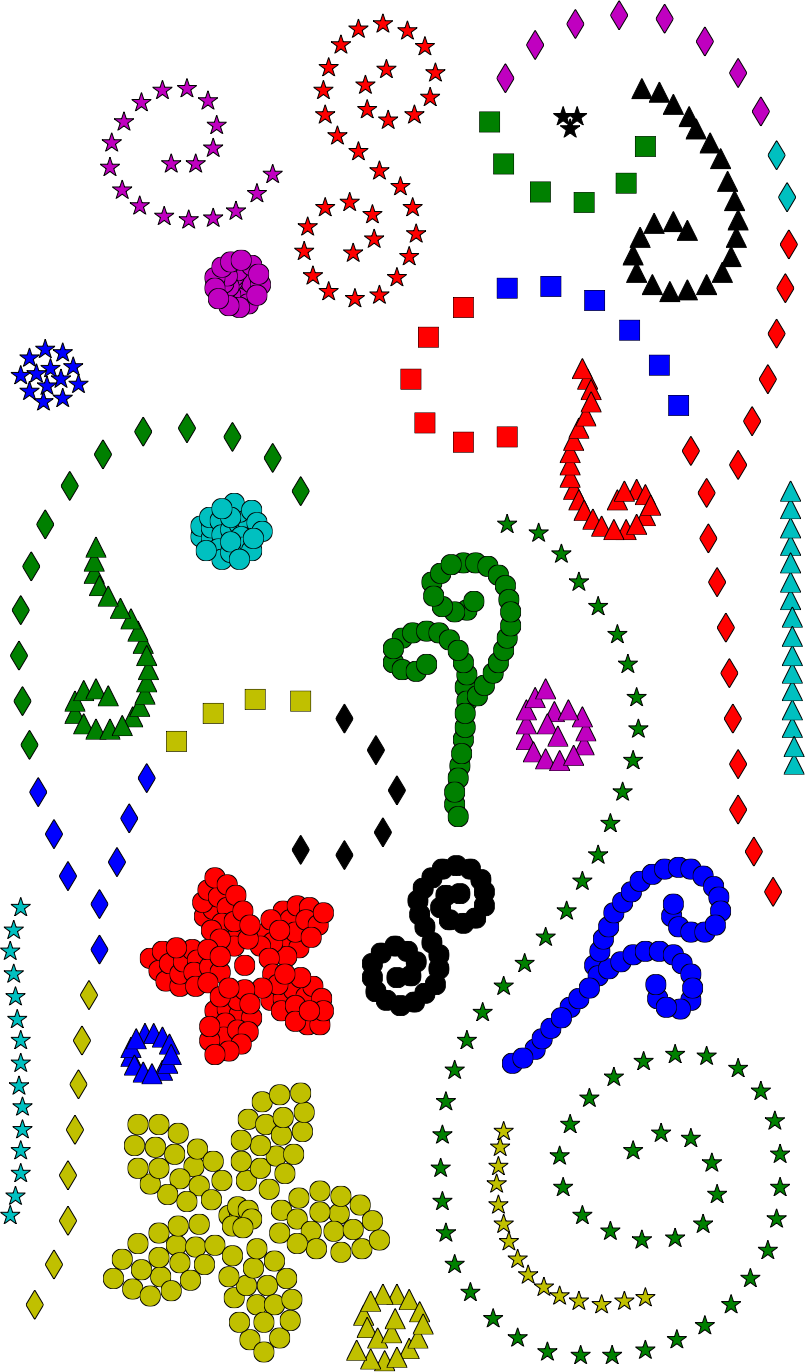}}
    }
    \subfloat[MMC]{
      \fbox{\includegraphics[width=.22\textwidth]{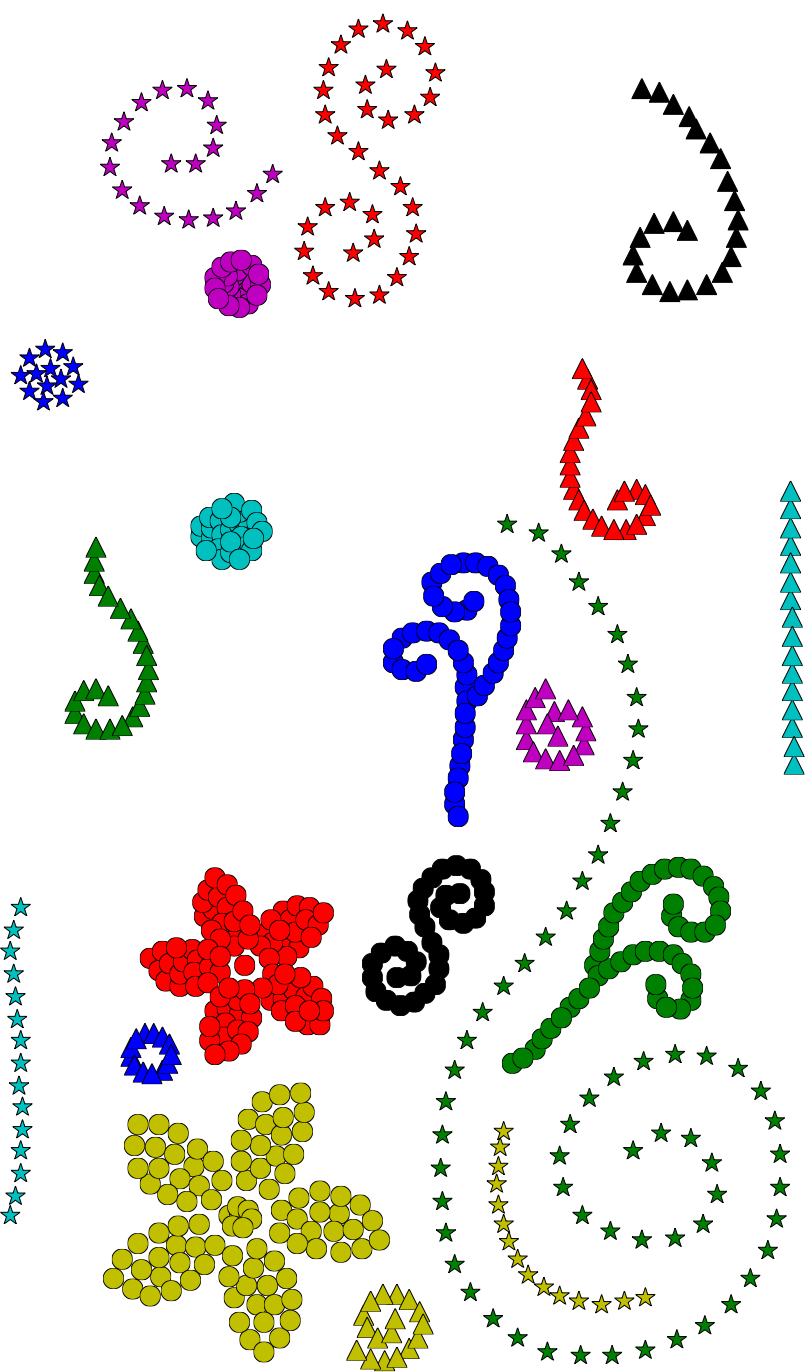}}
    }
  }

  \subfloat[Shapes drawn against noise. Shapes are respectively plotted in red and in black on the top and bottom rows.]{
    \begin{minipage}{.98\textwidth}
    \centerline{
      \fbox{\includegraphics[width=.22\textwidth]{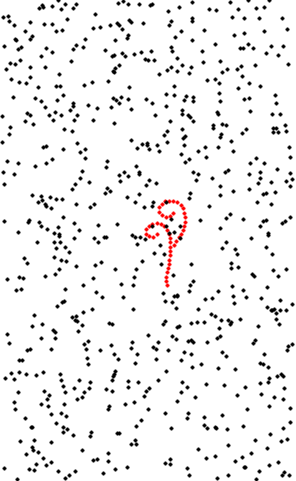}}
      \fbox{\includegraphics[width=.22\textwidth]{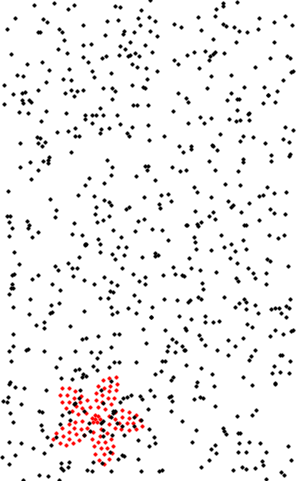}}
      \fbox{\includegraphics[width=.22\textwidth]{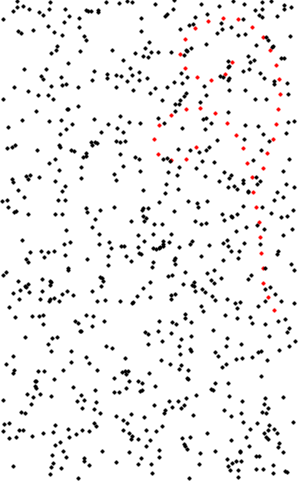}}
      \fbox{\includegraphics[width=.22\textwidth]{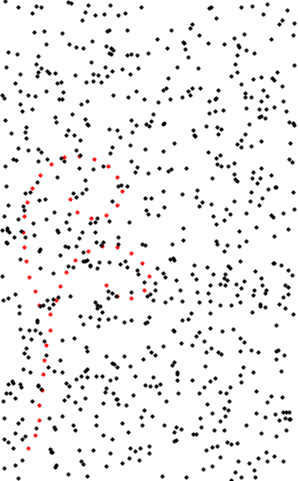}}
    }
    \vspace{.05in}
    \centerline{
      \fbox{\includegraphics[width=.22\textwidth]{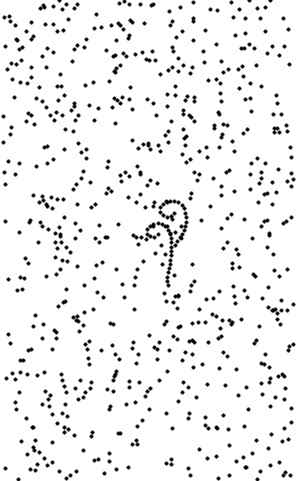}}
      \fbox{\includegraphics[width=.22\textwidth]{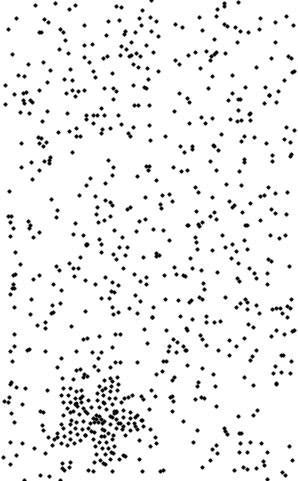}}
      \fbox{\includegraphics[width=.22\textwidth]{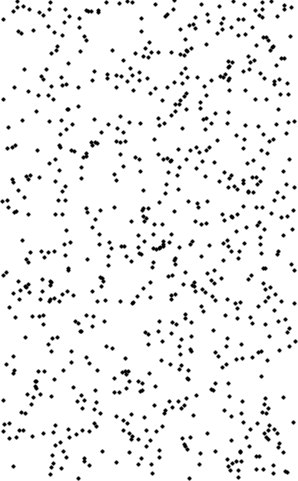}}
      \fbox{\includegraphics[width=.22\textwidth]{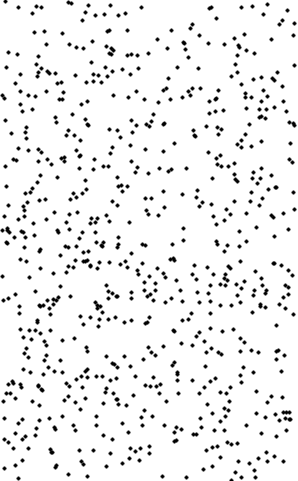}}
    }
    \end{minipage}
    \label{fig:expDots-noise}
  }

  \caption{In this example, maximal meaningful components correctly describe the organization of the points configuration. Only small or less denser figures are discarded. Indeed, meaningful components are clearly perceived in noise while non-meaningful components are not.}
  \label{fig:expDots}
\end{figure*}

Our results are compared with Felzenszwalb and Huttenlocher' algorithm (denoted by FH in the following) and with Mean Shift~\cite{comaniciu02,fukunaga75}. Mean Shift performs a non-parametric density estimation (using sliding windows) and finds its local maxima. Clusters are determined by what Comaniciu and Meer call ``bassins of attraction''~\cite{comaniciu02}: points are assigned to a local maximum following an ascendent path along the density gradient \footnote{code available at \url{http://www.mathworks.com/matlabcentral/fileexchange/10161-mean-shift-clustering}}.

Figure~\ref{fig:expDotsPedroAndMeanShift} present an experiment were FH and Mean Shift are used, respectively, to cluster the dataset in Figure~\ref{fig:dots}. Different runs were performed, by varying the kernel/scale size. Clearly, results are suboptimal in both cases. Both algorithms share the same main disadvantage: a global scale must be chosen a priori. Such a strategy is unable to cope with clusters of different densities and spatial sizes. Choosing a fine scale causes to correctly detect dense clusters at the price of oversplitting less denser ones. On the contrary, a coarser scale corrects the oversplitting of less denser clusters but introduces undersplitting for the denser ones.

\begin{figure*}[ht]
  \subfloat[Felzenszwalb and Huttenlocher' algorithm results at different scales.]{
    \begin{minipage}{.98\textwidth}
      \centerline{
	\fbox{\includegraphics[width=.22\textwidth]{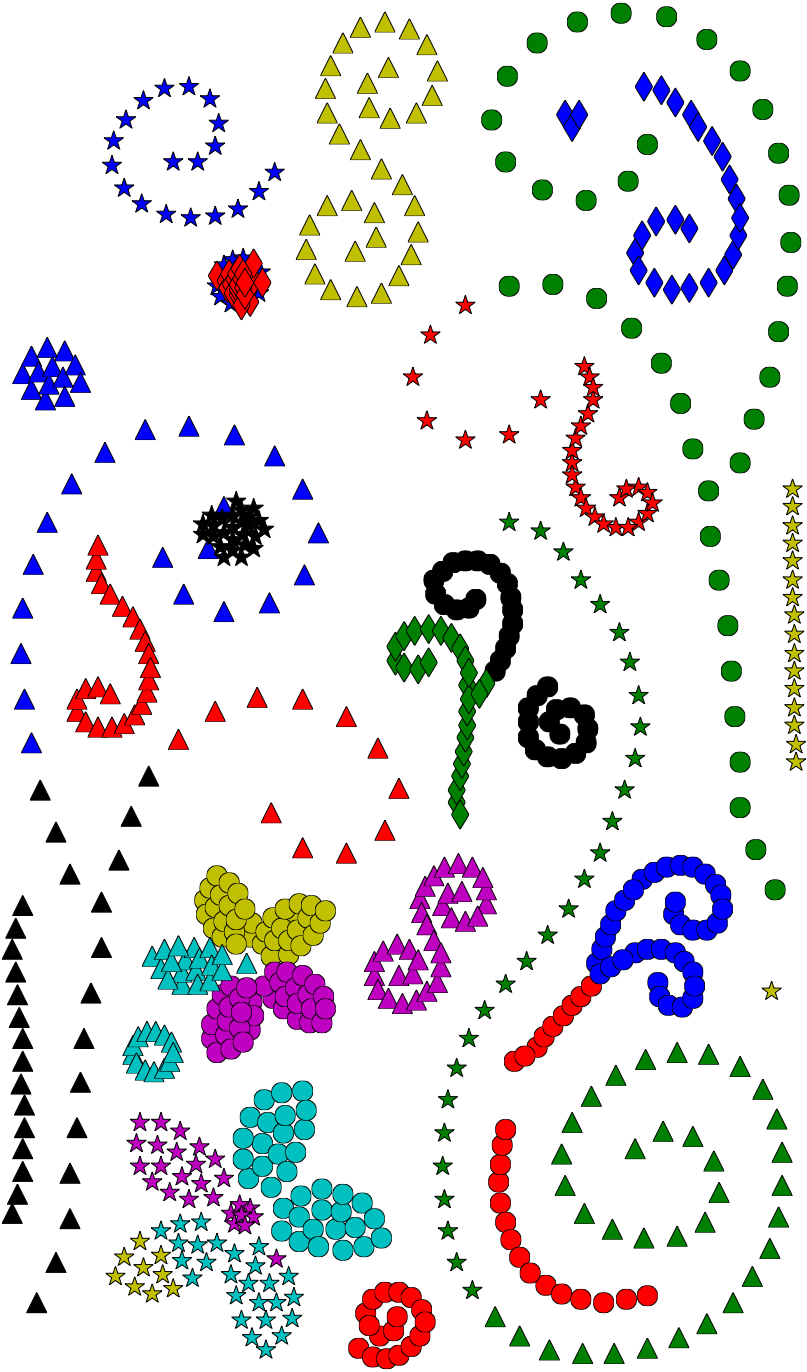}}
	\fbox{\includegraphics[width=.22\textwidth]{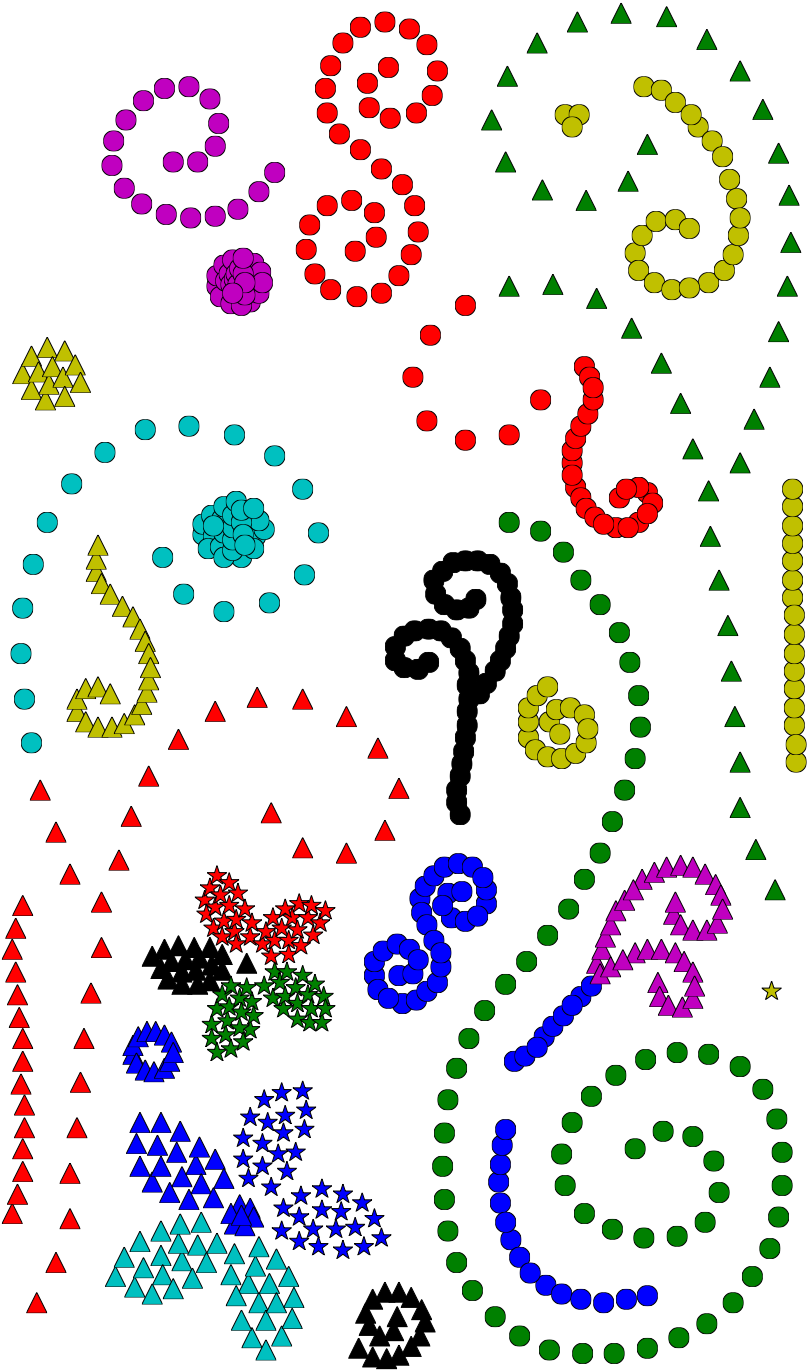}}
	\fbox{\includegraphics[width=.22\textwidth]{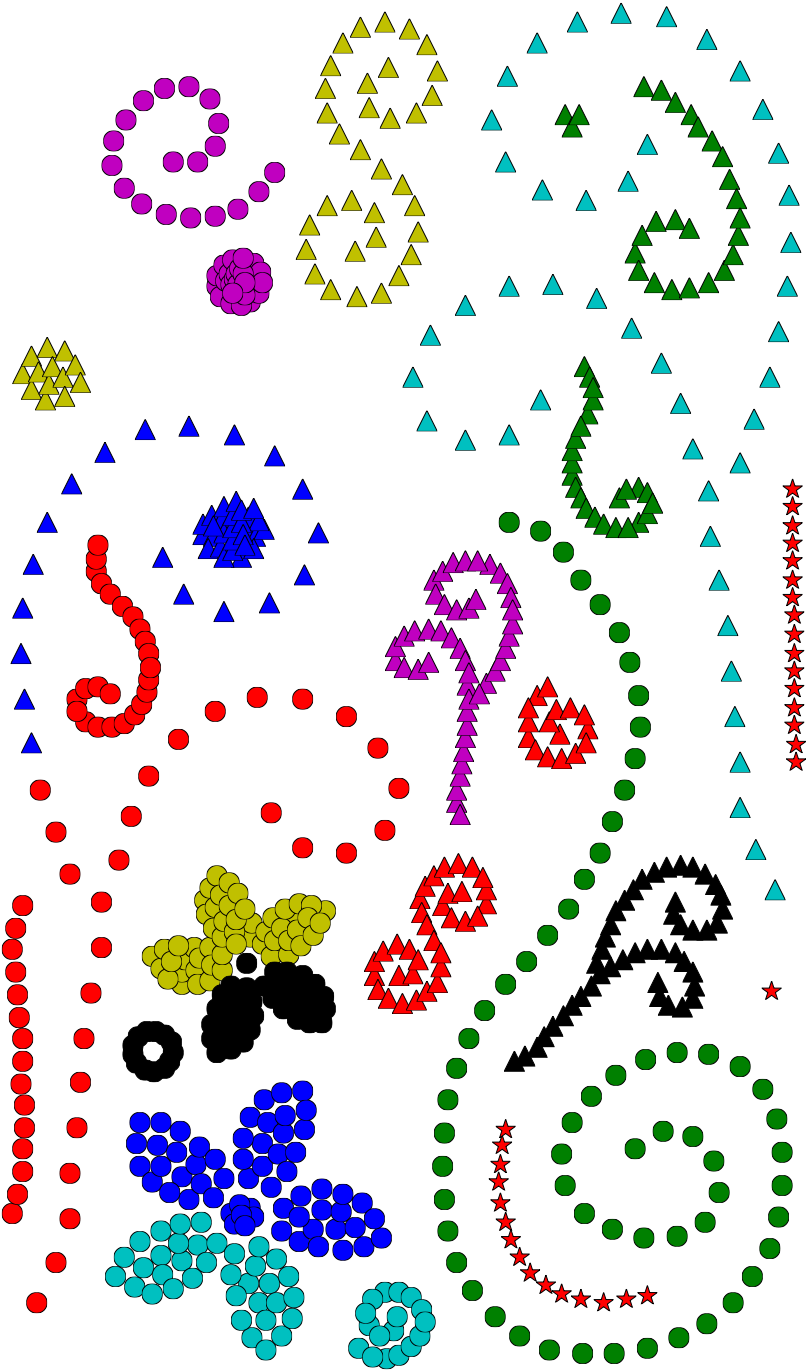}}
	\fbox{\includegraphics[width=.22\textwidth]{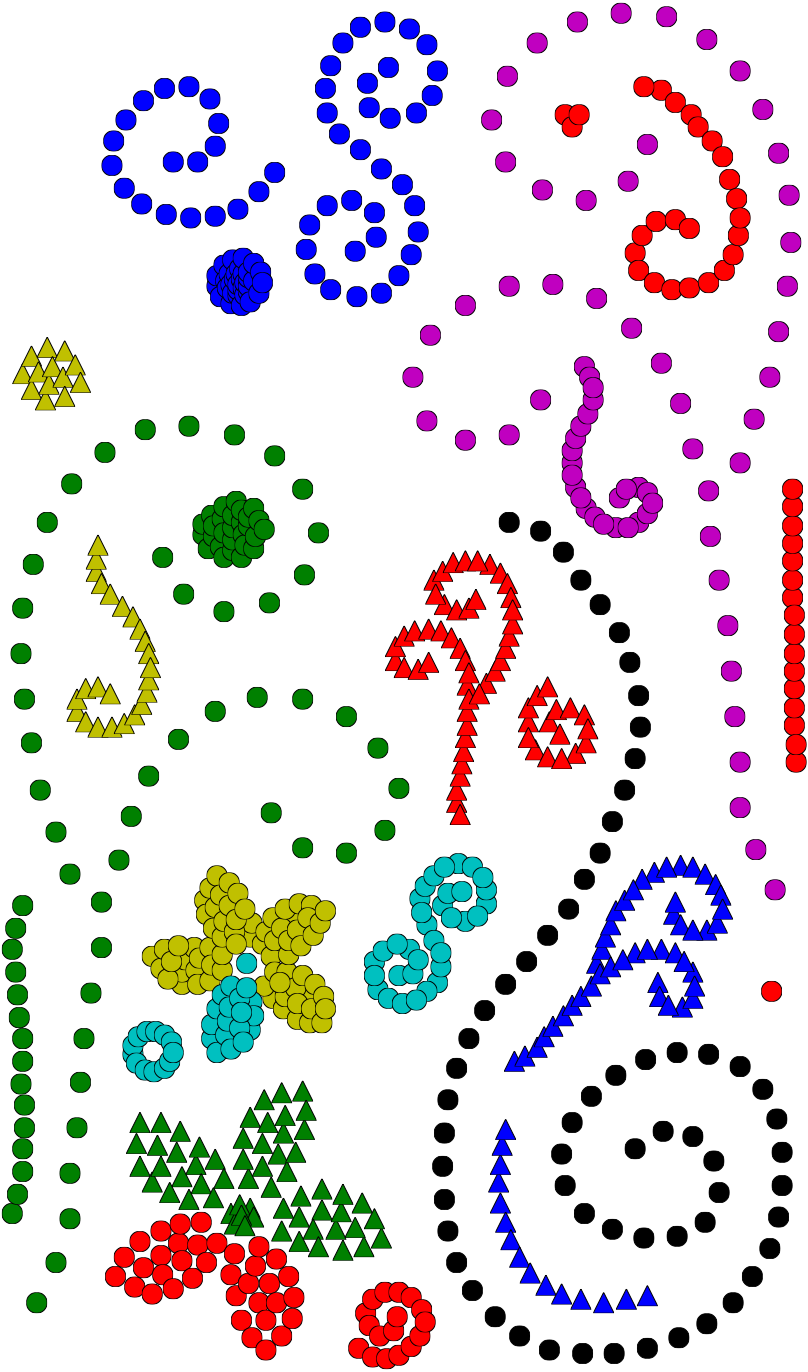}}
      }
    \end{minipage}

    \label{fig:expDotsPedro}
  }
  
  \subfloat[Mean Shift results with different values of kernel sizes.]{
    \begin{minipage}{.98\textwidth}
      \centerline{
	\fbox{\includegraphics[width=.22\textwidth]{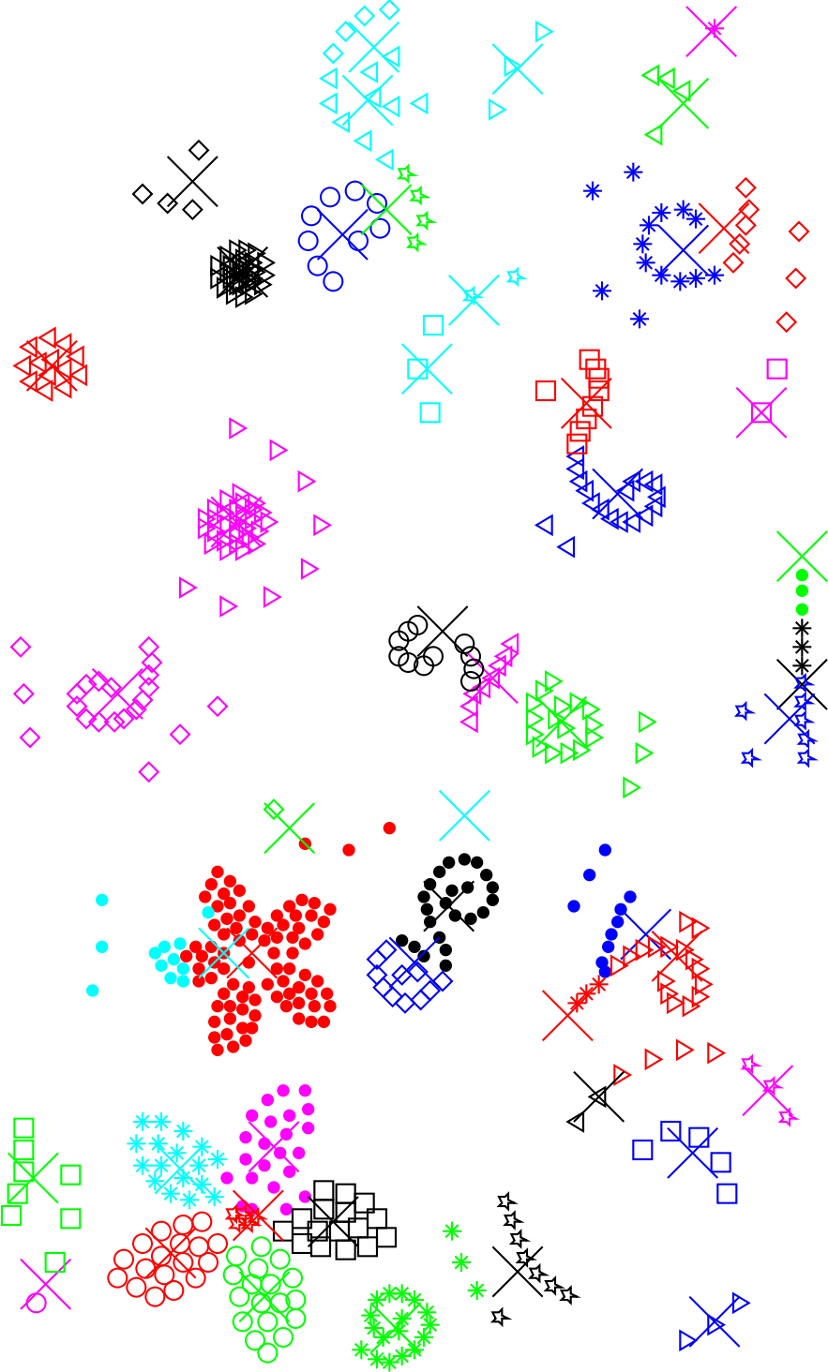}}
	\fbox{\includegraphics[width=.22\textwidth]{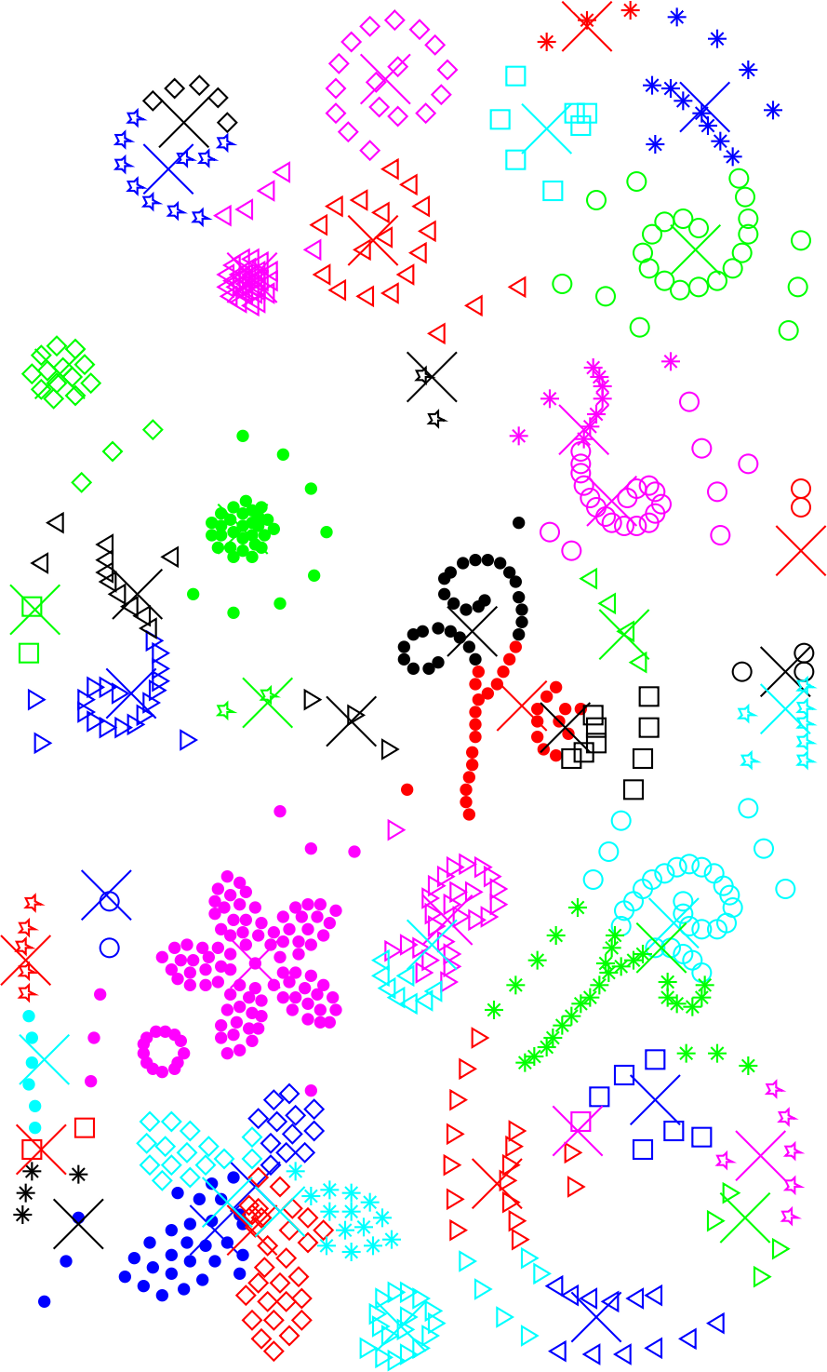}}
	\fbox{\includegraphics[width=.22\textwidth]{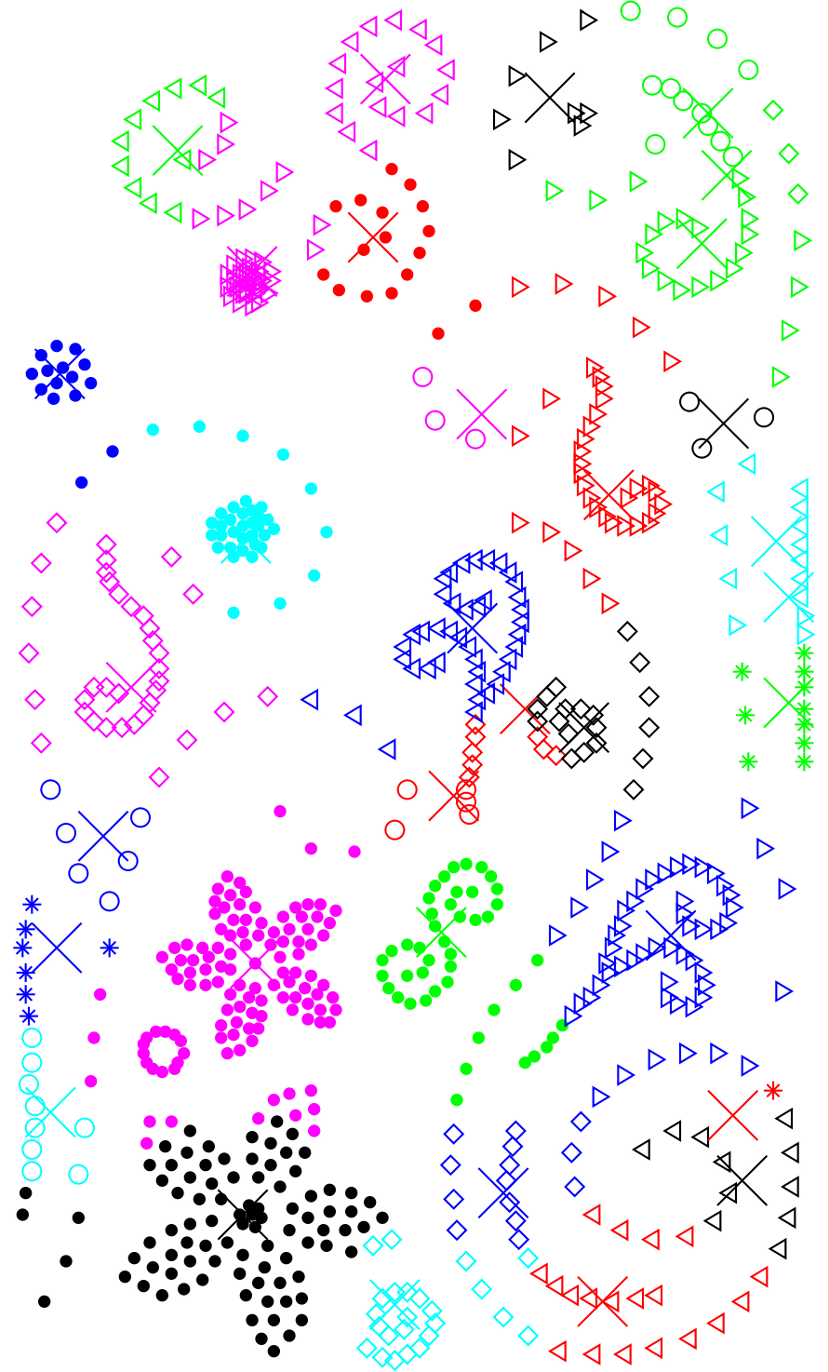}}
	\fbox{\includegraphics[width=.22\textwidth]{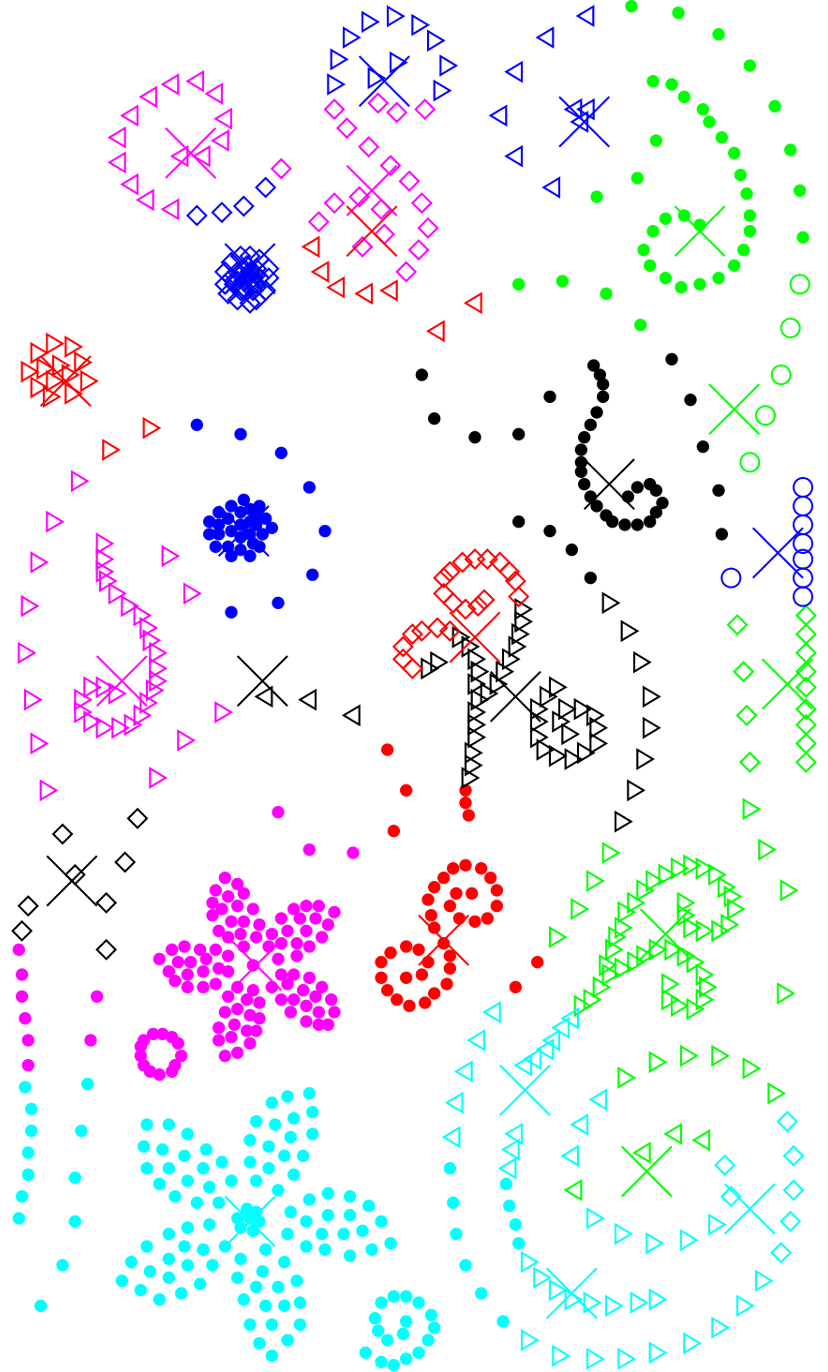}}
      }
    \end{minipage}
    \label{fig:expDotsMeanShift}
  }
  
  \caption{The same points configuration as in Figure~\ref{fig:dots}. At all scales, both algorithms fail to correctly detect the organization. Under and oversplitting occur in all cases.}
  \label{fig:expDotsPedroAndMeanShift}
\end{figure*}

\section{Handling MST Instability}
\label{sec:stabilization}

A seemingly obvious but interesting phenomenon occurs when noise is added to clustered data. Suppose we have data with two well separated clusters. In the absence of noise, it exists an MST edge linking both clusters. If noise is added to the data, the edge would probably disappear and be replaced by a sequence of edges. The length of the original linking edge is larger than the length of the endpoints of the sequence. The direct consequence is an increase in the NFA of the two clusters. Depending on the magnitude of that increase, the clusters might potentially be split into several proximal meaningful components. See Figure~\ref{fig:expStability}.

In short terms, noise affects the ideal topology of the MST. The oversplitting phenomenon can be corrected by iterating the following steps:
\begin{enumerate}
 \item detecting the meaningful clustered forest, \label{conceptualReplacementStart}
 \item add the union of points in the meaningful clustered forest to a new input dataset,
 \item remove the points in the meaningful clustered forest and replace them with noise, \label{conceptualReplacement}
 \item iterate until convergence, \label{conceptualReplacementEnd}
 \item re-detect the meaningful clustered forest on the new noise-free dataset built along all iterations.
\end{enumerate}
The MST of the set formed by merging the meaningful clustered forests from all iterations has the right topology. In other words this MST resembles the MST of the original data without noise. Then, detection of meaningful clustered forest can be performed without major trouble. We say that these detections form a stabilized meaningful clustered forest.

The above method implicitly contains a key issue in step~\ref{conceptualReplacement}. Detected points must be replaced with others which have a completely different distribution (i.e. the background distribution) but must occupy the same space. Figure~\ref{fig:holes} contains an example of the need for such a strong requirement. Pieces of background data might become ``disconnected, or to be precise connected by artificially created new edges. In one dimension, these holes are easily contracted, but when the dimensionality increases the contracting scheme gains more and more degrees of freedom.

\begin{figure*}[ht]
  \centering
  \begin{tabular}{@{\hspace{4pt}}c@{\hspace{4pt}}c@{\hspace{4pt}}}
    \begin{footnotesize}Input data\end{footnotesize} &
    \begin{footnotesize}MCF\end{footnotesize} \tabularnewline
    \includegraphics[width=.4\textwidth]{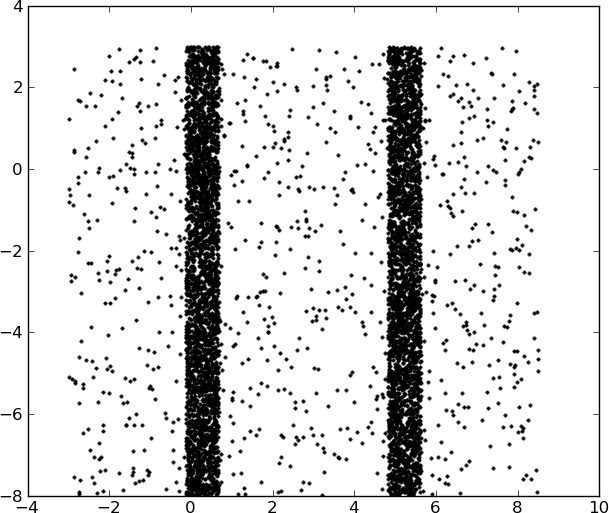} &
    \includegraphics[width=.4\textwidth]{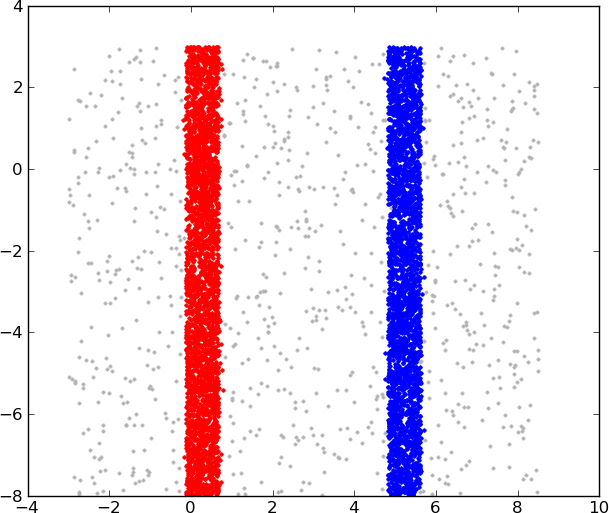} \tabularnewline
    \begin{footnotesize}Non-clustered data\end{footnotesize} &
    \begin{footnotesize}MST of non-clustered data\end{footnotesize} \tabularnewline
    \includegraphics[width=.4\textwidth]{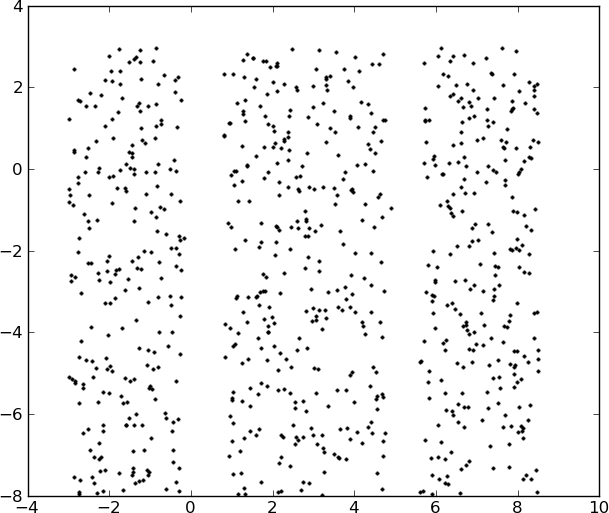} &
    \includegraphics[width=.4\textwidth]{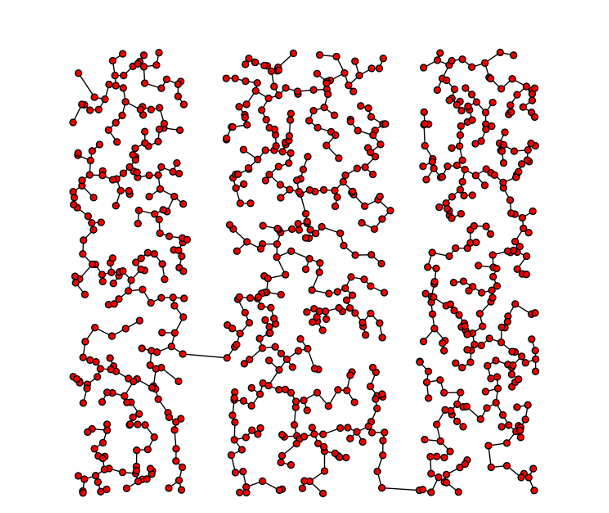}
  \end{tabular}

  \caption{Removing PMC can generate artifacts, i.e. holes, in the remaining data. These holes might create edges in the MST of the non-clustered data, that certainly violate the background model.}
  \label{fig:holes}
\end{figure*}

This noise filling procedure can be achieved by using the Voronoi diagram~\cite{voronoiSurvey} of the original point set. In the Voronoi diagram, each point lies on a different cell. To remove a point amounts to emptying a cell. Then the set of empty cells can be used as a sampling window to draw points from the background model. Notice that this procedure is actually generic since the Voronoi tesselation can be generalized to higher dimensional metric spaces~\cite{voronoiSurvey}.

The process simulates replacing detected components with noise from the background model. Due to the same nature of the Voronoi diagram, the process is not perfect: in the final iteration, the resulting point set is quasi but not exactly distributed according to the background model. A small bias is introduced, causing a few spurious detections in the MCF. To correct this issue it suffices to set $\eps = 10^{-2}$, as these detections have NFAs slightly lower than one and actual detections have really low NFAs. Of course this new thresholding could be avoided if a more accurate flling procedure was used.

Algorithm~\ref{algo:stabilize} illustrates steps~\ref{conceptualReplacementStart} to~\ref{conceptualReplacementEnd} of the correcting method. An example is shown in Figure~\ref{fig:stabilizeJain}, where four iterations are required until convergence.

\begin{algorithm}[ht]
\caption{Stabilize point set $X$ returning the set $\mathcal{F}$ of non-background points.}
\label{algo:stabilize}
\begin{algorithmic}[1]
  
  \STATE $\mathcal{F} \leftarrow \emptyset$
  \STATE $\mathcal{V} \leftarrow$ cells from Voronoi diagram of point set $X$ intersected with the minimum rectangle enclosing $X$.
  \STATE $X' \leftarrow X$
  \STATE $\mathcal{M} \leftarrow$ meaningful clustered forest of $X'$
  
  \WHILE{$\mathcal{M} \neq \emptyset$}
    
    \STATE $\mathcal{V}' \leftarrow \emptyset$
    \STATE $X' \leftarrow \emptyset$
    \FORALL{$C \in \mathcal{M}$}
      \FORALL{$\vect{p} \in C$}
	\STATE add $V \in \mathcal{V}$ to $\mathcal{V}'$ such that $\vect{p} \in V$.
	\IF{$\vect{p} \in X$}
	  \STATE add $\vect{p}$ to $X'$.
	  \STATE add $\vect{p}$ to $\mathcal{F}$.
	\ENDIF
      \ENDFOR
    \ENDFOR
    \STATE $\displaystyle a \leftarrow \sum_{V \in \mathcal{V}} \mathrm{area}(V)$
    \STATE $\displaystyle a_{\mathcal{M}} \leftarrow \sum_{V \in \mathcal{V}'} \mathrm{area}(V)$
    \STATE $\displaystyle n_{\mathcal{M}} \leftarrow \sum_{C \in \mathcal{M}} |C|$
    \STATE $\displaystyle n \leftarrow a_{\mathcal{M}} \cdot (|X| - N_{\mathcal{M}}) / (a - a_{\mathcal{M}})$
    \STATE $B \leftarrow$ draw $n$ points $\vect{q}_i$, $1 \leq i \leq n$, from the background model such that $(\exists V \in \mathcal{V}')\ q_i \in V$.
    \STATE $X' \leftarrow X' \cup B$
    \STATE $\mathcal{M} \leftarrow$ meaningful clustered forest of $X'$
    
  \ENDWHILE
\end{algorithmic}
\end{algorithm}

\begin{figure*}
  \centering
  \subfloat[In each iteration, the MCF is detected and the cells on the Voronoi diagram corresponding to points in the MCF are emptied and filled with points distributed according to the background model. In the fourth iteration, no MCF is detected and thus the algorithm stops.]{
  \begin{tabular}{@{\hspace{0pt}}m{.1in}@{\hspace{4pt}}m{.23\textwidth}@{\hspace{4pt}}m{.23\textwidth}@{\hspace{4pt}}m{.23\textwidth}@{\hspace{4pt}}m{.23\textwidth}@{\hspace{0pt}}}
    & \centering{\begin{footnotesize}Input data\end{footnotesize}} & \centering{\begin{footnotesize}MCF\end{footnotesize}} & \centering{\begin{footnotesize}MCF area\end{footnotesize}} & \centering{\begin{footnotesize}Replaced MCF\end{footnotesize}} \tabularnewline
    \begin{sideways}\begin{footnotesize}Iteration 1\end{footnotesize}\end{sideways} &
    \includegraphics[width=.23\textwidth]{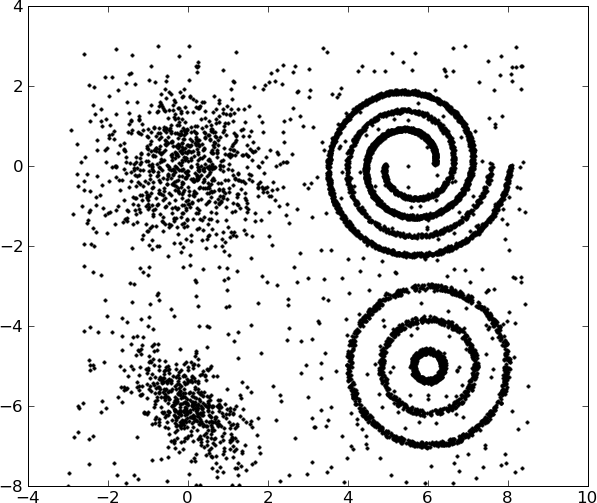} &
    \includegraphics[width=.23\textwidth]{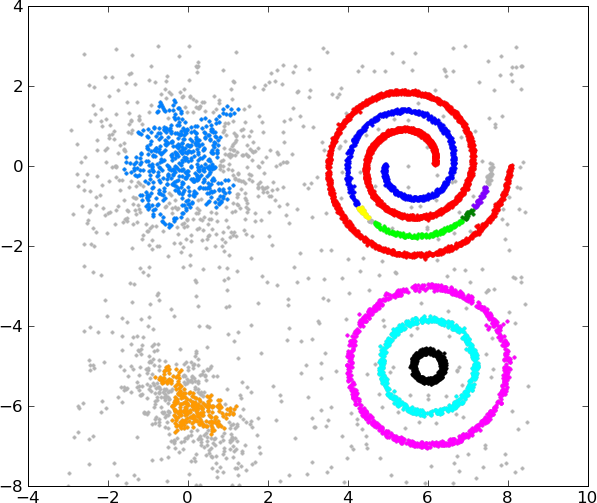} &
    \includegraphics[width=.23\textwidth]{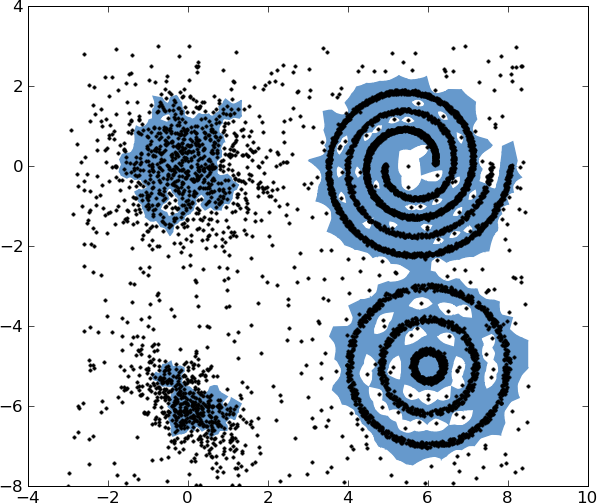} &
    \includegraphics[width=.23\textwidth]{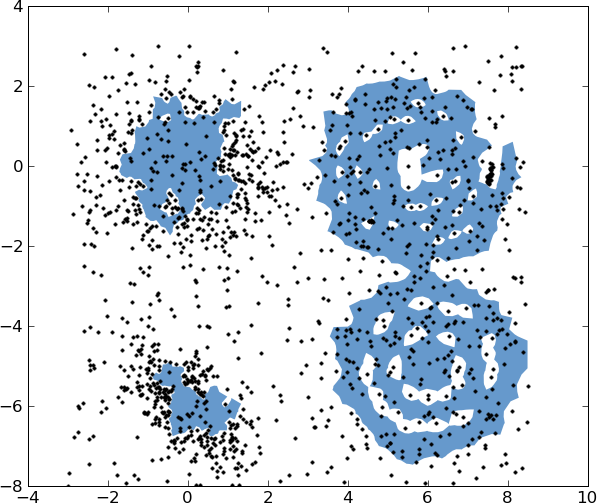} \tabularnewline
    
    \begin{sideways}\begin{footnotesize}Iteration 2\end{footnotesize}\end{sideways} &
    \includegraphics[width=.23\textwidth]{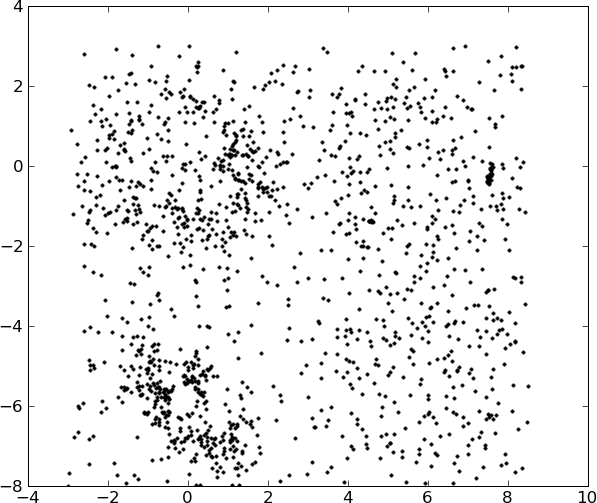} &
    \includegraphics[width=.23\textwidth]{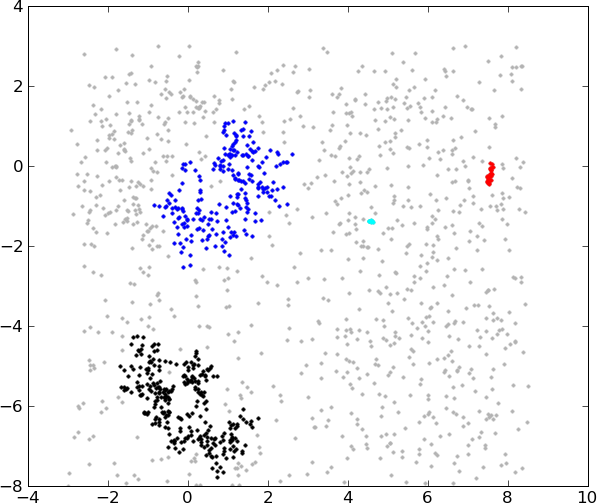} &
    \includegraphics[width=.23\textwidth]{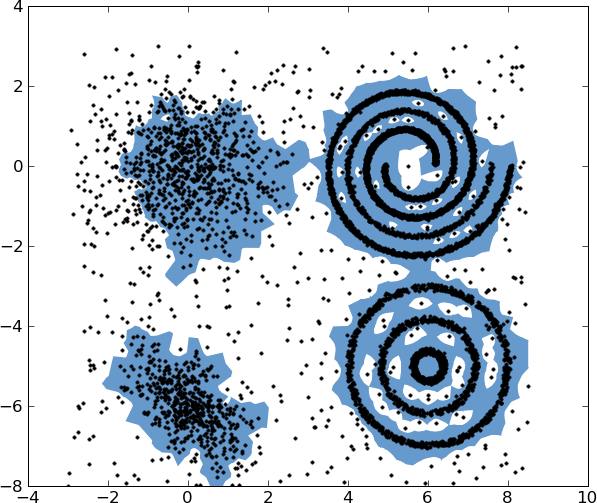} &
    \includegraphics[width=.23\textwidth]{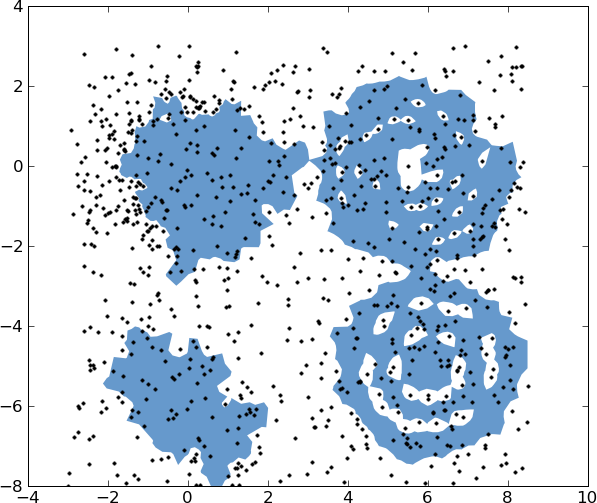} \tabularnewline
    
    \begin{sideways}\begin{footnotesize}Iteration 3\end{footnotesize}\end{sideways} &
    \includegraphics[width=.23\textwidth]{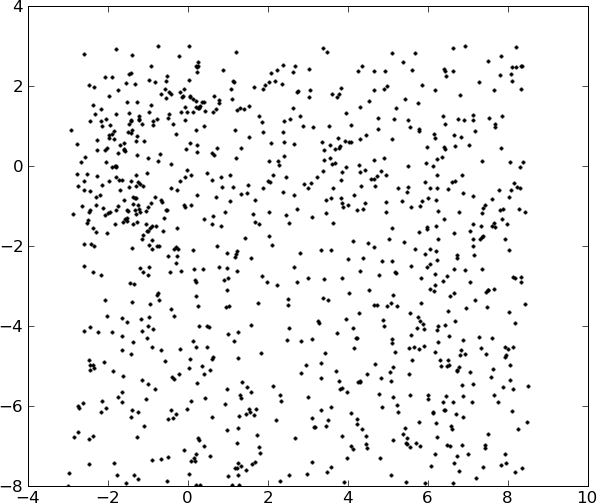} &
    \includegraphics[width=.23\textwidth]{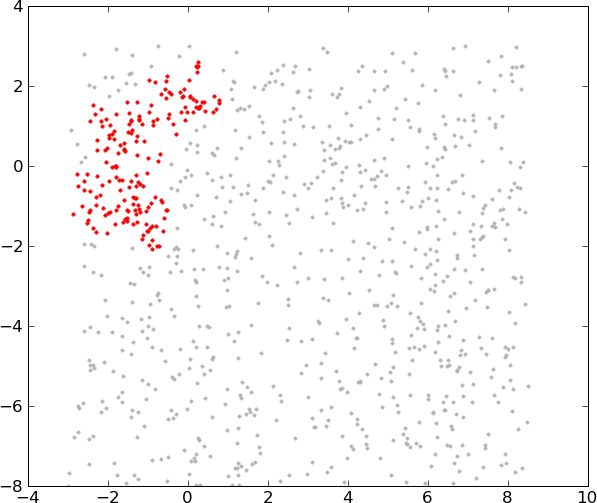} &
    \includegraphics[width=.23\textwidth]{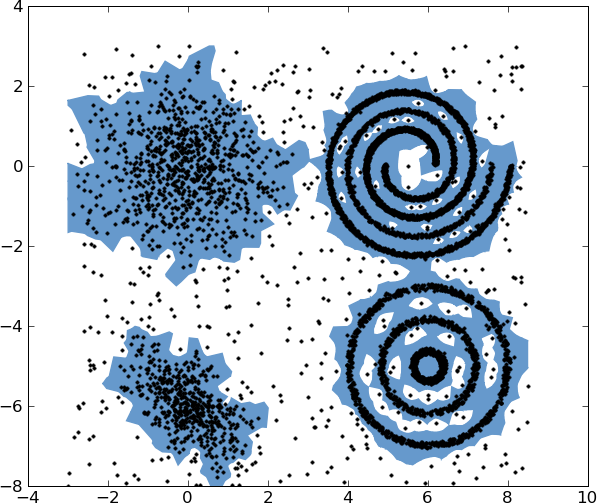} &
    \includegraphics[width=.23\textwidth]{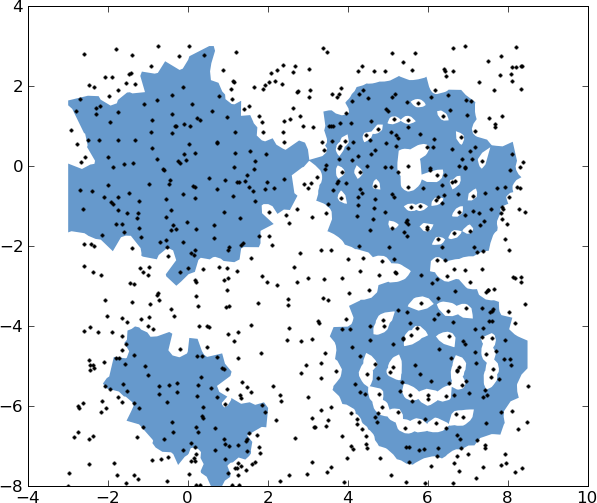} \tabularnewline
    
    \begin{sideways}\begin{footnotesize}Iteration 4\end{footnotesize}\end{sideways} &
    \includegraphics[width=.23\textwidth]{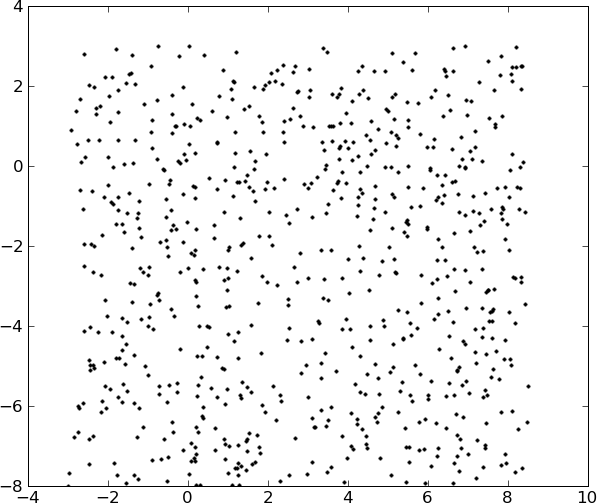} &
    \includegraphics[width=.23\textwidth]{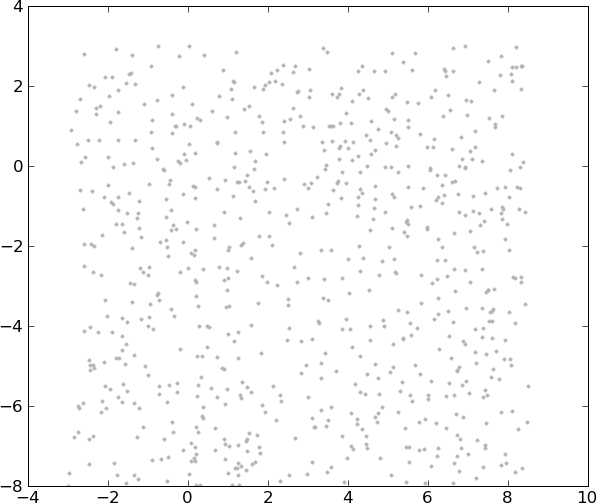} \tabularnewline
  \end{tabular}
  }
  
  \subfloat[Left, original Voronoi diagram. Right, resulting non-background points in red.]{
  \begin{minipage}{.98\textwidth}
    \centerline{
    \includegraphics[width=.3\textwidth]{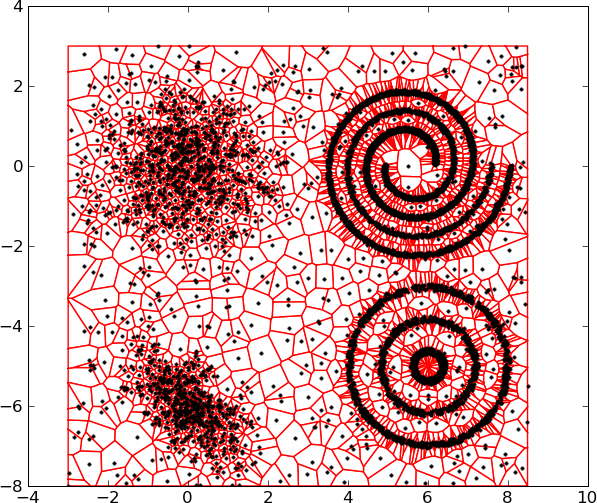}
    \includegraphics[width=.3\textwidth]{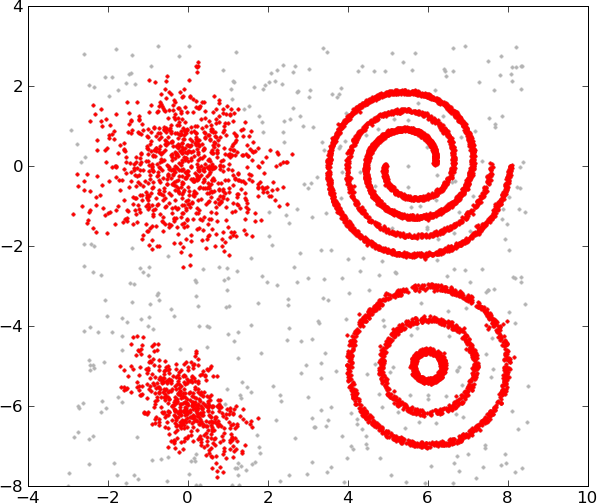}
    }
  \end{minipage}
  }
  
  \caption{Iteratively detecting the MCF and replacing it with points from the background model, converges and separates background from non-background data.}
  \label{fig:stabilizeJain}
\end{figure*}

Figure~\ref{fig:expStability} shows a second example of the stabilization process, followed by the detection of the stabilized meaningful clustered forest. The NFAs of the detected components are also included. The very low attained NFAs, account for the success of the procedure.

\begin{figure*}
  \centering{
  \begin{tabular}{cc}
    \subfloat[Input data]{\includegraphics[width=.45\textwidth]{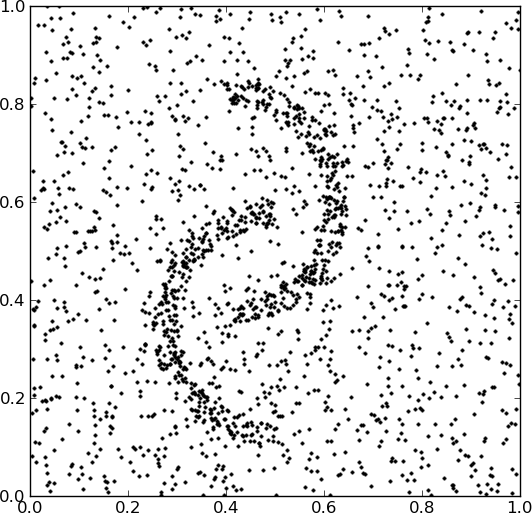}\label{fig:halfRingsUnstable}} &
    \subfloat[Desired clustering]{\includegraphics[width=.45\textwidth]{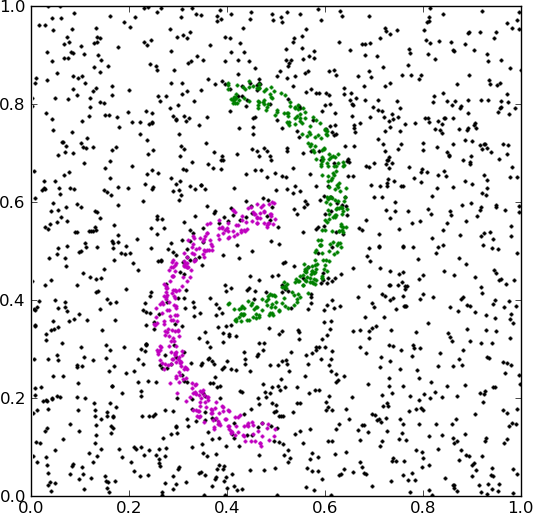}\label{fig:halfRingsUnstable-desired}} \tabularnewline
    \subfloat[Meaningful clustered forest]{
      \includegraphics[width=.45\textwidth]{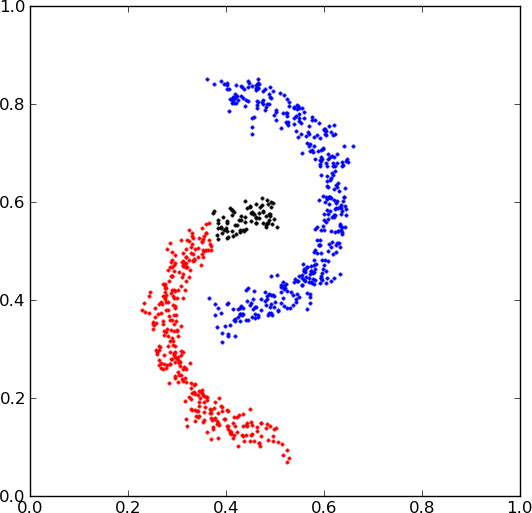}\label{fig:halfRingsUnstable-maximalMeaningful}
      \put(-70, 30){\begin{footnotesize}$\NFA \approx 10^{-11.2}$\end{footnotesize}}
      \put(-45, 38){\vector(-1, 0){50}}
      \put(-58, 130){\begin{footnotesize}$\NFA \approx 10^{-7.5}$\end{footnotesize}}
      \put(-33, 128){\vector(-2, -1){25}}
      \put(-140, 105){\begin{footnotesize}$\NFA \approx 10^{-7.2}$\end{footnotesize}}
      \put(-115, 103){\vector(3, -1){30}}
    } &
    \subfloat[Stabilized clustered forest]{
      \includegraphics[width=.45\textwidth]{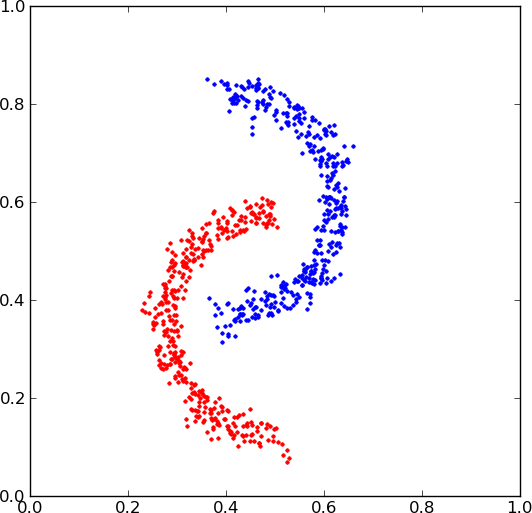}\label{fig:halfRingsUnstable-stabilizedMaximalMeaningful}
      \put(-70, 30){\begin{footnotesize}$\NFA \approx 10^{-55.8}$\end{footnotesize}}
      \put(-45, 38){\vector(-1, 0){50}}
      \put(-60, 128){\begin{footnotesize}$\NFA \approx 10^{-29.3}$\end{footnotesize}}
      \put(-35, 128){\vector(-2, -1){25}}
    }
  \end{tabular}
  }

  \caption{Noise might affect the stability of the meaningful clustered forest, causing to oversplit the clusters. Algorithm~\ref{algo:stabilize} converges in two iterations. Then, one can detect the meaningful clustered forest among non-background points, yielding a stabilized meaningful clustered forest.}
  \label{fig:expStability}
\end{figure*}

\section{The Masking Challenge}
\label{sec:masking}

In 2009, Jain~\cite{jain09} stated that no available algorithm could cluster the dataset in Figure~\ref{fig:jainImpossible3} and obtain the result in Figure~\ref{fig:jainImpossible3-desired}.
The dataset is interesting because it brings to light a new phenomenon: a cluster with many points can ``dominate'' the scene and hide other clusters that could be meaningful.

A similar behavior occurs in vanishing point detection, as pointed out by Almansa~\etal~\cite{almansa03}. A vanishing point is a point in an image to which parallel line segments not frontoparallel appear to converge; in some sense one can regard this point as a collection of such parallel line segments. Sometimes this procedure will still miss some weak vanishing points which are ``masked'' by stronger vanishing points composed of much more segments. These may not be perceived at first sight, but only if we manage to unmask them by getting rid of the ``clutter'' in one way or another. Almansa~\etal propose to eliminate these detected vanishing points and look for new vanishing points among the remaining line segments.

In our case, this very same approach cannot be followed. Masked clusters are not completely undetected but partially detected. Removing such cluster parts and re-detecting would cause oversegmentation. We propose instead to only remove the most meaningful proximal component and then iterate. The process ends when the masking phenomenon disappears, that is:
\begin{itemize}
 \item when there are no unclustered points, or
 \item no MCF is detected.
\end{itemize}
Algorithm~\ref{algo:unmasking} shows a detail of this unmasking scheme. Summarizing, first non-background points are detected using the stabilization process in Algorithm~\ref{algo:stabilize} and then the unmasking process takes place.

The detection of unmasked MMCs in Figure~\ref{fig:jainImpossible3-unmaskedMaximalMeaningful} correct all masking issues. Moreover, they are extremely similar to the desired clustering in Figure~\ref{fig:jainImpossible3-desired}. The difference is that clusters absorb background points that are within of near them. Indeed, these background points are statistically indistinguishable from the points from the cluster that absorbs them.

\begin{algorithm}
\caption{Compute the unmasked meaningful clustered forest $\mathcal{U}$ from the point set $\mathcal{F}$ of non-background points.}
\label{algo:unmasking}
\begin{algorithmic}[1]
  
  \STATE $\mathcal{U} \leftarrow \emptyset$
  
  \WHILE{$\mathcal{M} \neq \emptyset$}
    \STATE $\mathcal{M} \leftarrow$ stabilized meaningful clustered forest of $\mathcal{F}$
    \IF{$\displaystyle |F| = \sum_{C \in \mathcal{M}} |C|$}
      \STATE $\forall C \in \mathcal{M}$, add $C$ to $\mathcal{U}$.
    \ELSE
      \STATE $\displaystyle C_{\min} \leftarrow \argmin_{C \in \mathcal{M}} \NFA(C)$
      \FORALL{$\vect{p} \in C_{\min}$}
	\STATE remove $p$ from $\mathcal{F}$
      \ENDFOR
      \STATE add $C_{\min}$ to $\mathcal{U}$.
    \ENDIF
  \ENDWHILE
\end{algorithmic}
\end{algorithm}

From a total number of 7000 points in Figure~\ref{fig:jainImpossible3}, the outer spiral (in orange in in Figure~\ref{fig:jainImpossible3-desired}) has 2514 points, i.e. almost 36\% of the points. The detection of the unmasked MCF in Figure~\ref{fig:jainImpossible3-unmaskedMaximalMeaningful} correct all masking issues. Moreover, they are extremely similar to the desired clustering in Figure~\ref{fig:jainImpossible3-desired}. The difference is that clusters absorb background points that are within of near them. Indeed, these background points are statistically indistinguishable from the points from the cluster that absorbs them.

\begin{figure*}
  \centering
  \subfloat[In each iteration, the MCF is detected and the most meaningful component is removed from the dataset. In the sixth iteration, all points are clustered and thus the algorithm stops.]{
  \begin{tabular}{@{\hspace{0pt}}c@{\hspace{4pt}}c@{\hspace{4pt}}c@{\hspace{4pt}}c@{\hspace{4pt}}c@{\hspace{4pt}}c@{\hspace{4pt}}c@{\hspace{0pt}}}
    & \begin{footnotesize}Iteration 1\end{footnotesize} & \begin{footnotesize}Iteration 2\end{footnotesize} & \begin{footnotesize}Iteration 3\end{footnotesize} & \begin{footnotesize}Iteration 4\end{footnotesize} & \begin{footnotesize}Iteration 5\end{footnotesize} & \begin{footnotesize}Iteration 6\end{footnotesize} \\
    \raisebox{.15in}{\begin{sideways}\begin{footnotesize}Input data\end{footnotesize}\end{sideways}} &
    \includegraphics[width=.15\textwidth]{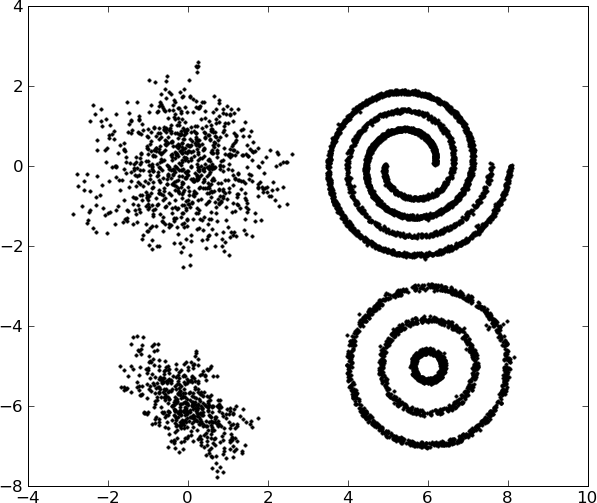} &
    \includegraphics[width=.15\textwidth]{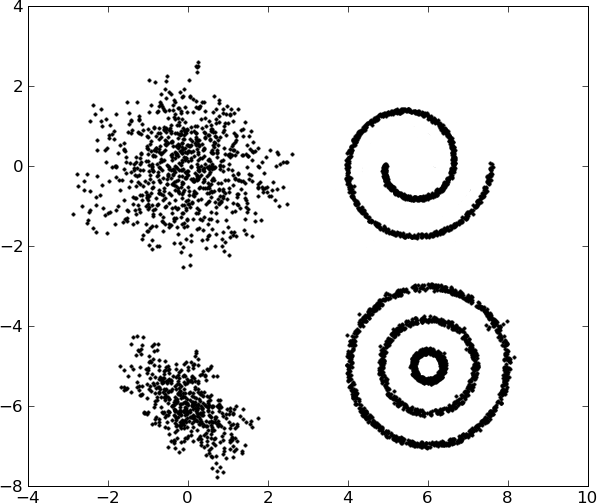} &
    \includegraphics[width=.15\textwidth]{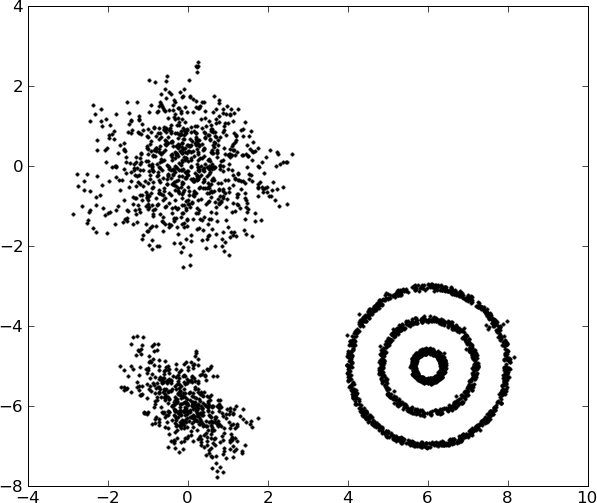} &
    \includegraphics[width=.15\textwidth]{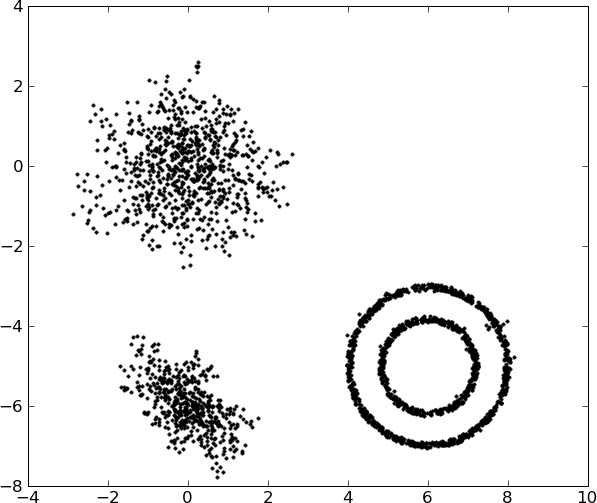} &
    \includegraphics[width=.15\textwidth]{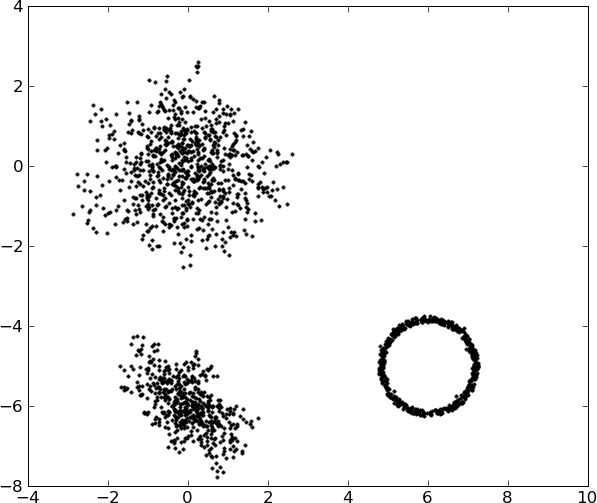} &
    \includegraphics[width=.15\textwidth]{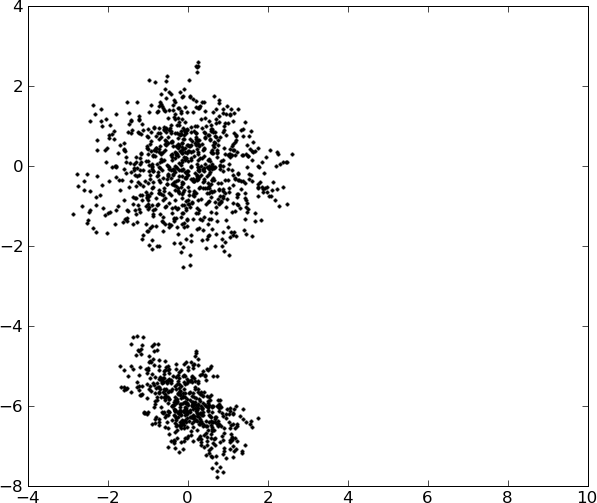} \\
    
    \raisebox{.32in}{\begin{sideways}\begin{footnotesize}MCF\end{footnotesize}\end{sideways}} &
    \includegraphics[width=.15\textwidth]{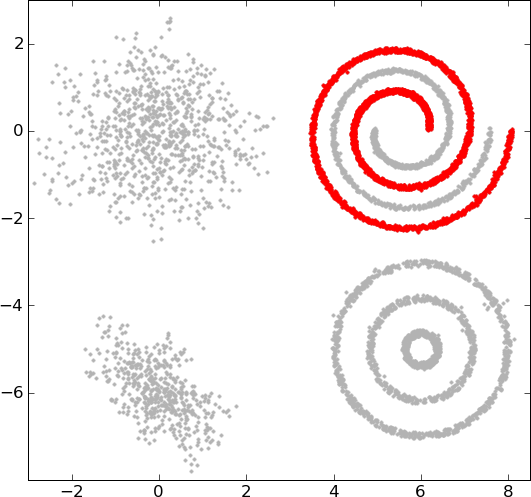} &
    \includegraphics[width=.15\textwidth]{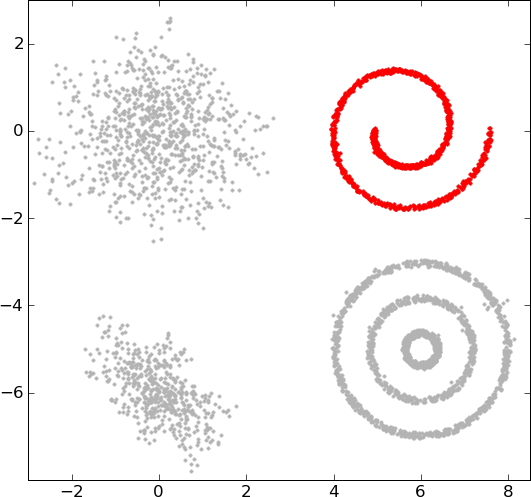} &
    \includegraphics[width=.15\textwidth]{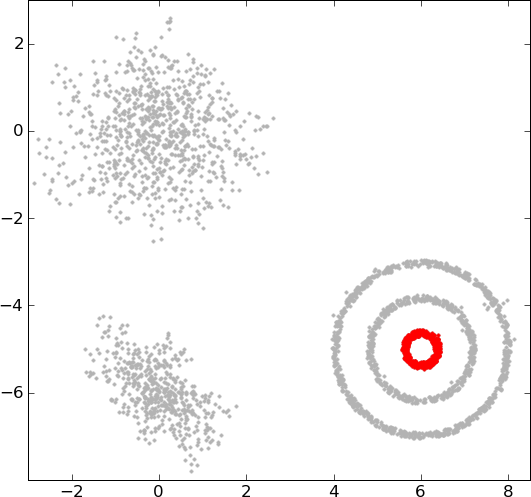} &
    \includegraphics[width=.15\textwidth]{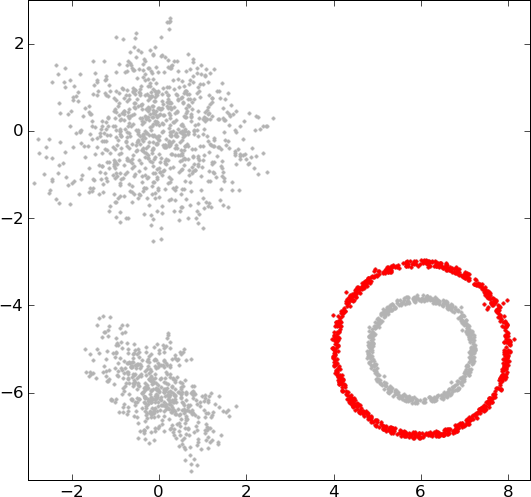} &
    \includegraphics[width=.15\textwidth]{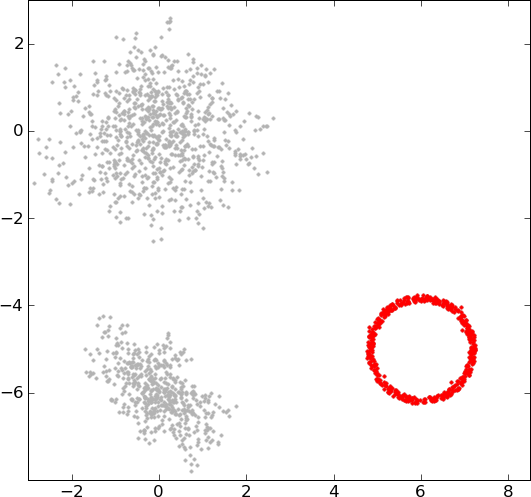} &
    \includegraphics[width=.15\textwidth]{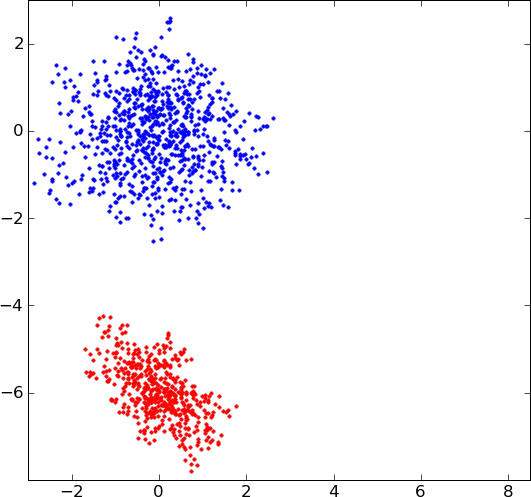}
  \end{tabular}
  }
  
  \begin{tabular}{@{\hspace{0pt}}c@{\hspace{4pt}}c@{\hspace{4pt}}c@{\hspace{0pt}}}
    \subfloat[Input data]{
      \includegraphics[width=.3\textwidth]{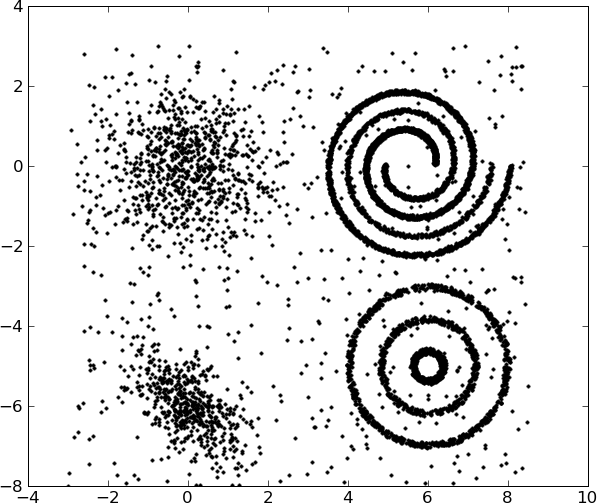}
      \label{fig:jainImpossible3}
    } &
    \subfloat[Desired clustering]{
      \includegraphics[width=.3\textwidth]{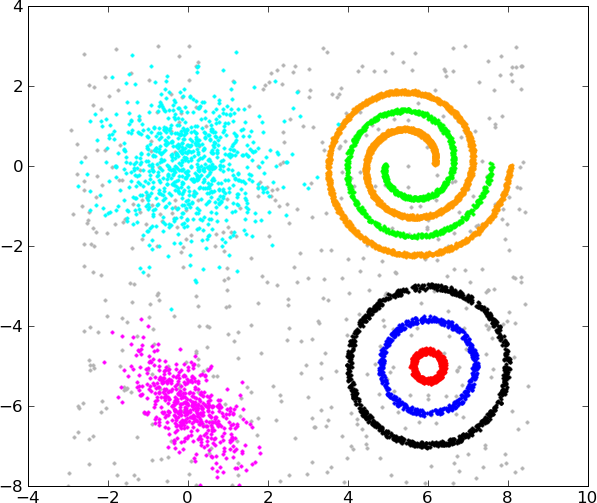}
      \label{fig:jainImpossible3-desired}
    } &s
    \subfloat[Unmasked MCF.]{
      \includegraphics[width=.3\textwidth]{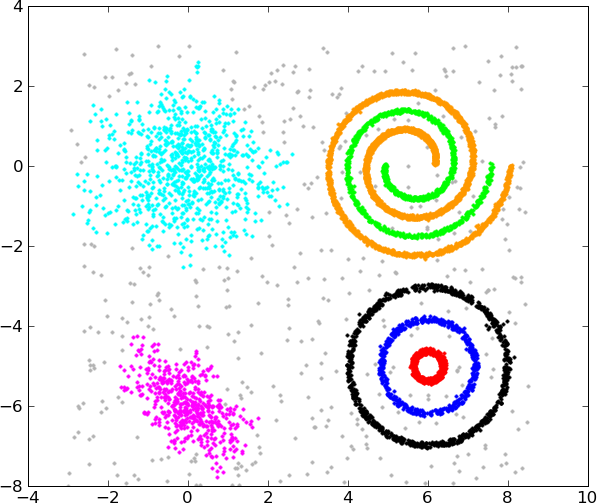}
      \label{fig:jainImpossible3-unmaskedMaximalMeaningful}
    }
  \end{tabular}
  
  \caption{Iteratively detecting the MCF and removing from the dataset its most meaningful component, converges and corrects the masking phenomenon. The detected MCF is extremely similar to the desired clustering. The difference is that clusters absorb background points that are within of near them.}
  \label{fig:unmaskJain}
\end{figure*}

\section{Three-dimensional point clouds}
\label{sec:3dPointClouds}

We tested the proposed algorithm with three-dimensional point clouds. We put two point clouds in the same scene at different positions, thus building two scenes in Figures~\ref{fig:clouds-scene1} and~\ref{fig:clouds-scene5}. In both cases uniformly distributed noise was artificially added. The skeleton hand and the bunny are formed by 3274 and by 3595, respectively. In Figure~\ref{fig:clouds-scene1}, 3031 noise points were added to total 9900 points. In Figure~\ref{fig:clouds-scene1}, 7031 noise points were added to total 13900 points and both shapes were positioned closer to each other and in such a way that no linear separation exist between them. In both cases the result is correct

In Figure~\ref{fig:clouds-scene1}, the MCF is oversplit but the stabilization process discussed in Section~\ref{sec:stabilization} corrects the issue. In Figure~\ref{fig:clouds-scene1}, although the same phenomenon is possible, it does not occur in this realization of the noise process.

\begin{figure*}
  \centering
  \begin{tabular}{@{\hspace{0pt}}m{.11in}@{\hspace{4pt}}m{.33\textwidth}@{\hspace{4pt}}m{.33\textwidth}@{\hspace{0pt}}}
    & \centering{\begin{footnotesize}View 1\end{footnotesize}} & \centering{\begin{footnotesize}View 2\end{footnotesize}} \tabularnewline
    \begin{sideways}\begin{footnotesize}Input data\end{footnotesize}\end{sideways} &
    \includegraphics[width=.33\textwidth]{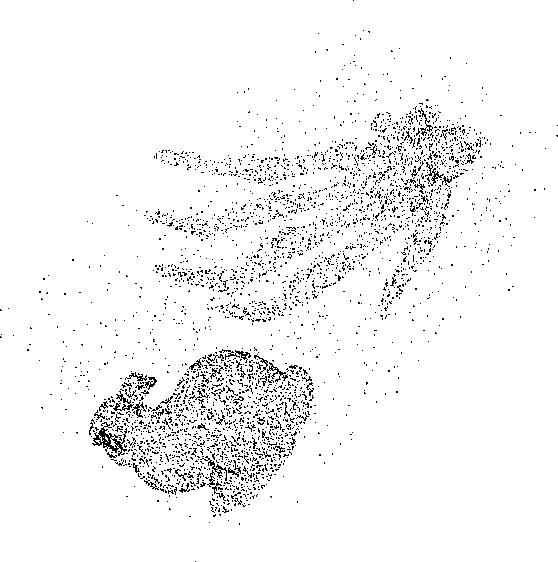} &
    \includegraphics[width=.33\textwidth]{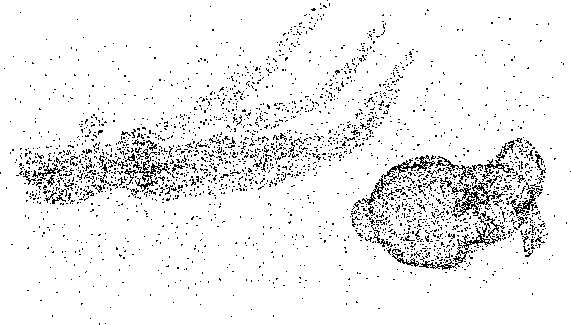} \tabularnewline
    
    \begin{sideways}\begin{footnotesize}Ground truth\end{footnotesize}\end{sideways} &
    \includegraphics[width=.33\textwidth]{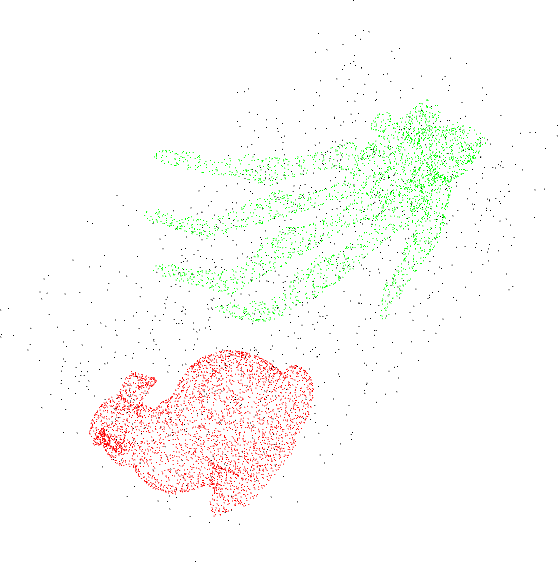} &
    \includegraphics[width=.33\textwidth]{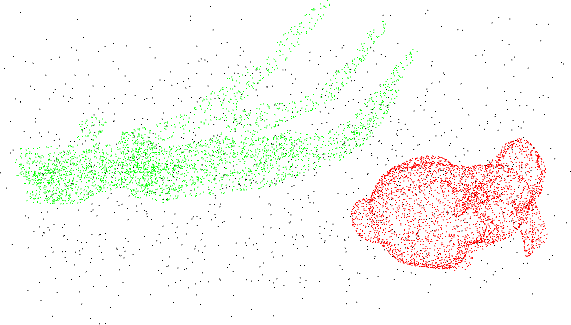} \tabularnewline

    \begin{sideways}\begin{footnotesize}MCF\end{footnotesize}\end{sideways} &
    \includegraphics[width=.33\textwidth]{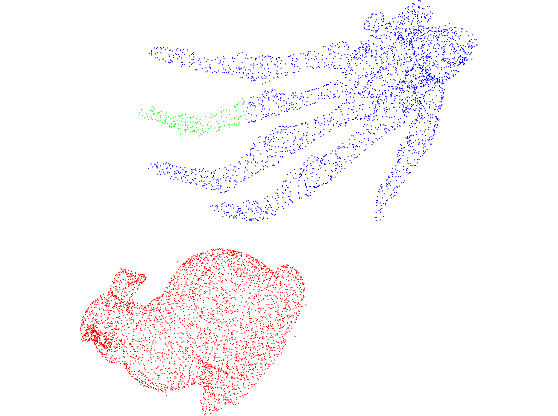} &
    \includegraphics[width=.33\textwidth]{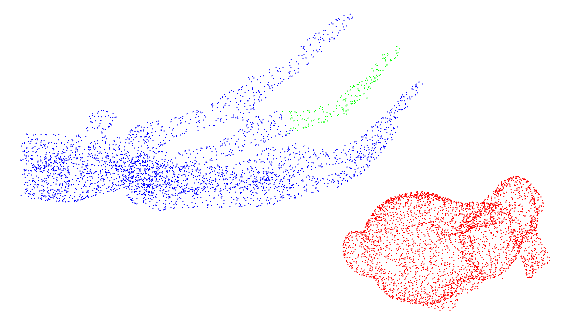} \tabularnewline
    
    \begin{sideways}\begin{footnotesize}Stabilized MCF\end{footnotesize}\end{sideways} &
    \includegraphics[width=.33\textwidth]{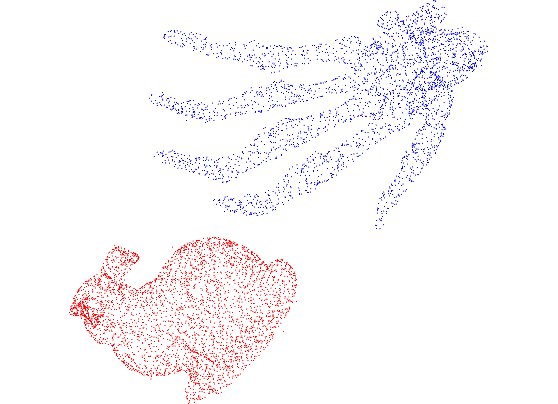} &
    \includegraphics[width=.33\textwidth]{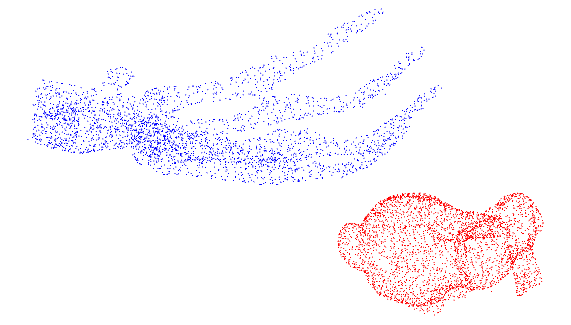} \tabularnewline
  \end{tabular}
  
  \caption{Two point clouds with artificially added noise. In this case, noise perturbed the MCF (see the finger of the skeleton hand). This effect is corrected by the stabilization process.}
  \label{fig:clouds-scene1}
\end{figure*}

\begin{figure*}
  \centering
  \begin{tabular}{@{\hspace{0pt}}m{.11in}@{\hspace{4pt}}m{.35\textwidth}@{\hspace{4pt}}m{.35\textwidth}@{\hspace{0pt}}}
    & \centering{\begin{footnotesize}View 1\end{footnotesize}} & \centering{\begin{footnotesize}View 2\end{footnotesize}} \tabularnewline
    \begin{sideways}\begin{footnotesize}Input data\end{footnotesize}\end{sideways} &
    \includegraphics[width=.35\textwidth]{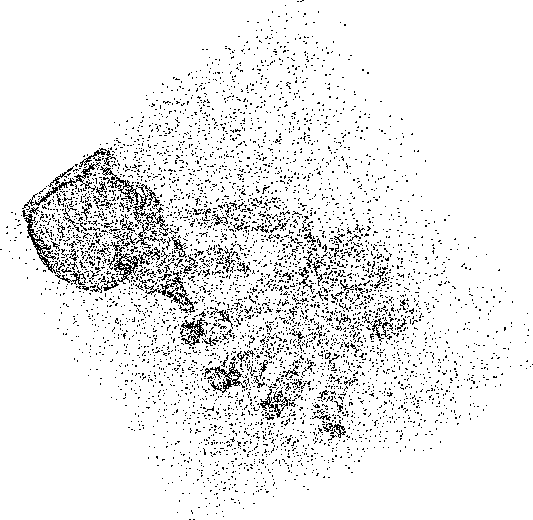} &
    \includegraphics[width=.35\textwidth]{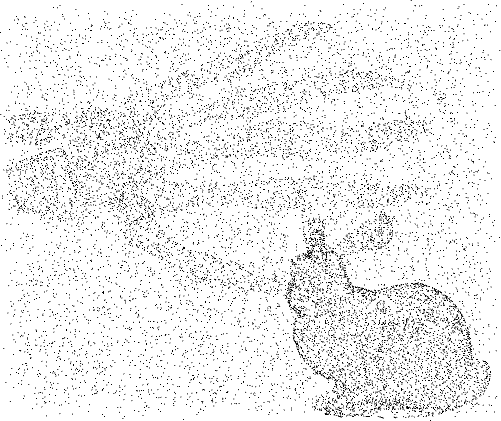} \tabularnewline
    
    \begin{sideways}\begin{footnotesize}Ground truth\end{footnotesize}\end{sideways} &
    \includegraphics[width=.35\textwidth]{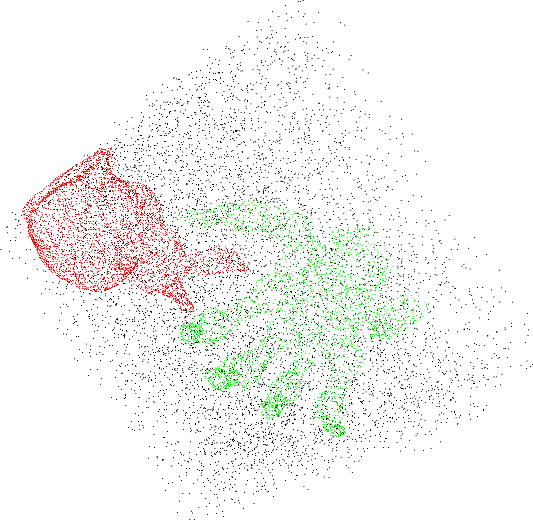} &
    \includegraphics[width=.35\textwidth]{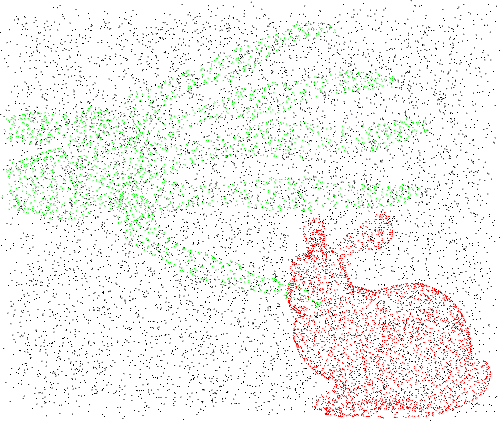} \tabularnewline

    \begin{sideways}\begin{footnotesize}MCF\end{footnotesize}\end{sideways} &
    \includegraphics[width=.35\textwidth]{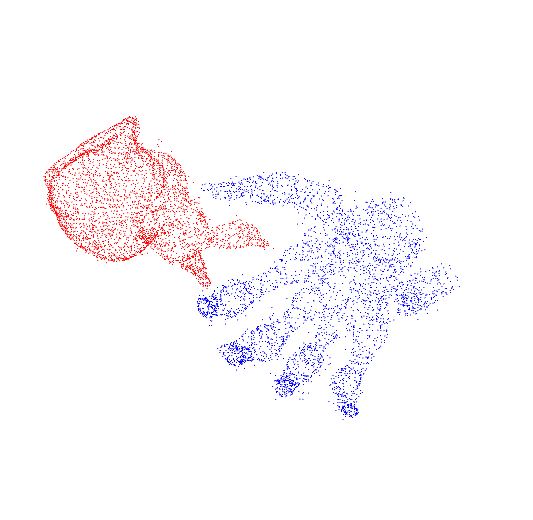} &
    \includegraphics[width=.35\textwidth]{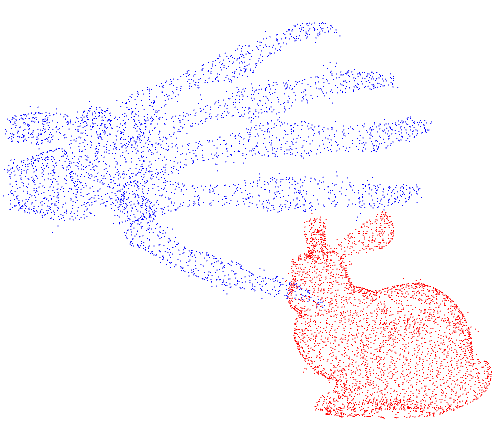} \tabularnewline
  \end{tabular}
  
  \caption{Two point clouds with artificially added noise. Both shapes are close to each other and are not linearly separable. The result of the stabilization process is omitted as detections do not change.}
  \label{fig:clouds-scene5}
\end{figure*}

\section{Final Remarks}
\label{sec:mstClusteringConclusions}

In this work we propose a new clustering method that can be regarded as a numerical method to compute the proximity gestalt. The method relies on analyzing edge distances in the MST of the dataset. The direct consequence is that our approach is fully parametric on the chosen distance.

The proposed method present several novelties over other MST-based formulations.
Some formulations have preference for compact clusters as they extract their clustering detection rule from characteristics that are not intrinsic to the MST. Our method only focuses on the length of the MST edges; hence, it does not present such preference.
Other formulations analyze the data at a fixed local scale, thus introducing a new method parameter. We have shown through examples that these local methods can fail when the input data has clusters with different sizes and densities. In these same examples, our method perform well without the need of introducing any extra parameter.

The method is also automatic, in the sense that only a single parameter is left to the user. This parameter has an intuitive interpretation as it controls the expected number of false detections. Moreover, setting it to $1$ is sufficient in practice.

Robustness to noise is an additional but essential feature of the method. Indeed, we have shown that the iterative application of our method can be used to treat noisy data, producing quality results.

We also studied the masking phenomenon in which a highly populated and salient cluster dominates the scene and inhibits the detection of less-populated, but still salient, clusters. The proposed method can be iteratively used to avoid such inhibitions from happening, yielding promising results.

As future work, it would be interesting to study the MST edge distribution under different point distributions. From the theoretical point of view, it can bring light to the method correctness. In practice, it would allow to replace the simulated background models by their analytical counterparts.
 
\appendix

\section{Proof of Proposition~\ref{prop:expectation}}
\label{sec:proofs}

The proof relies on the following classical lemma.
\begin{lemma}
  Let $X$ be a real random variable and let $F(x) = P(X \leq x)$ be the cumulative density function of $X$. Then for all $t \in (0, 1)$,
  \begin{equation}
    \Pr ( F(X) < t) \leq t
  \end{equation}
  \label{lem:classic}
\end{lemma}

\begin{proposition}
  The expected number of \meps-meaningful clusters in a random single-link hierarchy (i.e. issued from the background model) is lower than \meps.
\end{proposition}

\begin{proof}
  We follow the scheme of Proposition 1 from the work by Cao~\etal~\cite{cao2005}.
  Let $\Tree$ be random single-link hierarchy. For brevity let $M = |X| - 1$.
  Let $Z_i$ be a binary random variable equal to 1 if the random cluster $C_i \in \Tree$ is meaningful and 0 else.

  Let us denote by $\expectation(X)$ the expectation of a random variable $X$ in the a contrario model.
  We then have
  \begin{equation}
    \expectation \left( \sum_{i=1}^{M} Z_i \right) = \expectation \left( \expectation \left( \sum_{i=1}^{M} Z_i \ |\ M \right) \right) \text{.}
  \end{equation}
    Let $Y_i$ be a binary random variable equal to 1 if
    \begin{equation}
      M \cdot \Pr \Big( \omega_{\max}(\random{C}) < \omega_{\max}(\deterministic{C_i}) \ |\ \omega_{\max}(\father{\random{C}}) \sim \omega_{\max}(\father{{\deterministic{C_i}}}) \Big) < \eps
    \end{equation}
    and 0 else. Of course, $M$ is independent from the sets in $\Tree$. Thus, conditionally to $M = m$, the law of $\sum_{i=1}^{M} Z_i$ is the law of $\sum_{i=1}^{M} Y_i$.
    Let us reprise Equation~\ref{eq:rankStatistics} on p.~\pageref{eq:rankStatistics},
    \begin{multline}
      \Pr \Big( \omega_{\max}(\random{C}) < \omega_{\max}(\deterministic{C_i}) \ |\ \omega_{\max}(\father{\random{C}}) \sim \omega_{\max}(\father{{\deterministic{C}_i}}) \Big) \\
      = \mathrm{F}_{\Omega} \Big(  \omega_{\max}(\deterministic{C_i}),\  \omega_{\max}(\father{\deterministic{{C_i}}}) \Big)^{K_i} .
    \end{multline}
    By linearity of expectation,
  \begin{equation}
    \expectation \left( \sum_{i=1}^{M} Z_i \ |\ M = m \right) = \expectation \left( \sum_{i=1}^{m} Y_i \right) = \sum_{i=1}^{m} \expectation \left( Y_i \right) \text{.}
  \end{equation}
  Let us denote $\mathrm{F}_{\Omega} \Big(  \omega_{\max}(\deterministic{C_i}),\  \omega_{\max}(\father{\deterministic{{C_i}}}) \Big)$ by $\Pr(\deterministic{C_i})$.
  Since $Y_i$ is a Bernoulli variable,
  \begin{multline}
    \expectation(Y_i) = \Pr(Y_i = 1) = 
    \Pr \left( M \cdot \Pr(\deterministic{C_i})^{K_i} < \eps \right) \\
    = \sum_{k=0}^{\infty} \Pr \left( M \cdot \Pr(\deterministic{C_i})^{K_i} < \eps \ \Big|\ K_i = k \right) \cdot \Pr \left( K_i = k \right) \text{.}
    \label{eq:expectationYi}  
  \end{multline}
  We have assumed that $K_i$ is independent from $\Pr(\deterministic{C_i})$.
  Thus, conditionally to $K_i = k$, $\Pr(\deterministic{C_i})^{K_i} = \Pr(\deterministic{C_i})^{k}$.
  We have
  \begin{multline}
      \Pr \left( m \cdot \Pr \Big( \omega_{\max}(\random{C}) < \omega_{\max}(\deterministic{C_i}) \ |\ \omega_{\max}(\father{\random{C}}) \sim \omega_{\max}(\father{{\deterministic{C}_i}}) \Big) < \eps \right) \\
      \shoveleft= \Pr \left( m \cdot \mathrm{F}_{\Omega} \Big(  \omega_{\max}(\deterministic{C_i}),\  \omega_{\max}(\father{\deterministic{{C_i}}}) \Big)^k < \eps \right) \\
      \shoveleft= \Pr \left( m \cdot \Pr \Big( \Omega < \omega_{\max}(\deterministic{C_i}) \ |\ \omega_{\max}(\father{\random{C}}) \sim \omega_{\max}(\father{{\deterministic{C}_i}}) \Big)^k < \eps \right) \\
      \shoveleft= \Pr \left( \Pr \Big( \Omega < \max_{e \in \deterministic{C_i}}\omega(e) \ |\ \omega_{\max}(\father{\random{C}}) \sim \omega_{\max}(\father{{\deterministic{C}_i}}) \Big) < \left( \frac{\eps}{m} \right)^{1/k} \right) \\
      \shoveleft= \Pr \left( \max_{e \in E(\deterministic{C_i})} \Pr \Big( \Omega < \omega(e) \ |\ \omega_{\max}(\father{\random{C}}) \sim \omega_{\max}(\father{{\deterministic{C}_i}}) \Big) < \left( \frac{\eps}{m} \right)^{1/k} \right) \\
      = \prod_{e \in E(\deterministic{C_i})} \Pr \left( \Pr \Big( \Omega < \omega(e)  \ |\ \omega_{\max}(\father{\random{C}}) \sim \omega_{\max}(\father{{\deterministic{C}_i}}) \Big) < \left( \frac{\eps}{m} \right)^{1/k} \right). 
  \end{multline}
  The last equality follows from the conditional independence assumption. Now, using Lemma~\ref{lem:classic} and bearing in mind that the number of edges in $\deterministic{C_i}$ is $K_i = k$, yields
 \begin{equation}
      \Pr \left( m \cdot\Pr(\deterministic{C_i})^{k} < \eps \right)
       \leq \prod_{j = 1}^k \left( \frac{\eps}{m} \right)^{1/k} = \frac{\eps}{m}.
  \end{equation}
  We can now use this bound in Equation~\ref{eq:expectationYi}:
  \begin{multline}
    \expectation(Y_i) = \sum_{k=0}^{\infty} \Pr \left( M \cdot \Pr(\deterministic{C_i})^{K_i} < \eps \ \Big|\ K_i = k \right) \cdot \Pr \left( K_i = k \right) \\
    \leq \frac{\eps}{m} \sum_{k=0}^{\infty} \Pr \left( K_i = k \right) = \frac{\eps}{m} \text{.}
  \end{multline}
  Hence,
  \begin{equation}
    \expectation \left( \sum_{i=1}^{M} Z_i \ |\ M = m \right) = \sum_{i=1}^m \expectation(Y_i) \leq \eps \text{.}
  \end{equation}
  This finally implies $\expectation \left( \sum_{i=1}^{M} Z_i \right) \leq \eps$, what means that the expected number of \meps-meaningful clusters is less than \meps.
\end{proof}

\section*{Acknowledgments}

The authors acknowledge financial support by CNES (R\&T Echantillonnage Irregulier DCT / SI / MO - 2010.001.4673), FREEDOM (ANR07-JCJC-0048-01), Callisto (ANR-09-CORD-003), ECOS Sud U06E01, ARFITEC (07 MATRH) and STIC Amsud (11STIC-01 - MMVPSCV) and the Uruguayan Agency for Research and Innovation (ANII) under grant PR-POS-2008-003.

\bibliographystyle{plain}
\bibliography{mtepper}

\end{document}